\documentclass[11pt]{article}
\usepackage{graphicx,color}
\usepackage[colorlinks=true, linkcolor=black, citecolor=black, urlcolor=black]{hyperref}
\usepackage[font=footnotesize,labelfont=bf]{caption}
\usepackage{authblk}
\usepackage{geometry}
\usepackage[normalem]{ulem}
\usepackage{algpseudocode,algorithm}
\usepackage{comment}
\usepackage[T1]{fontenc}       
\usepackage{lmodern}           
\usepackage[utf8]{inputenc}           
\usepackage{amsmath, amssymb, amsfonts, amsthm} 
\usepackage{mathtools}        
\usepackage{mathrsfs}         
\usepackage{bbm}
\usepackage{booktabs}

\newtheorem{proposition}[]{Proposition}

\newtheorem{assumption}[]{Assumption}

\theoremstyle{definition}

\usepackage[title]{appendix}

\usepackage{natbib}

\begin{document}
\title{
Scale adaptive and robust intrinsic dimension estimation\\
via optimal neighbourhood identification
}

\author{
   {Antonio} {Di Noia}\textsuperscript{1,2,*},
   {Iuri} {Macocco}\textsuperscript{3},
   {Aldo} {Glielmo}\textsuperscript{4},
   {Alessandro} {Laio}\textsuperscript{5},
   and {Antonietta} {Mira}\textsuperscript{2,6}
}

\date{\today}

\maketitle

\let\thefootnote\relax\footnotetext{
\textsuperscript{1}Seminar for Statistics, Department of Mathematics, ETH Zurich.
\textsuperscript{2}Faculty of Economics, Euler Institute, Università della Svizzera italiana.
\textsuperscript{3}Department of Translation and Language Sciences, Universitat Pompeu Fabra.
\textsuperscript{4}Banca d'Italia. The views and opinions expressed in this paper are those of the authors and do not necessarily reflect the official policy or position of Banca d'Italia.
\textsuperscript{5}International School for Advanced Studies (SISSA).
\textsuperscript{6}Department of Science and High Technology, University of Insubria.}

\begin{abstract}
The Intrinsic Dimension (ID) is a key concept in unsupervised learning and feature selection, as it is a lower bound to the number of variables which are necessary to describe a system. 
However, in almost any real-world dataset the ID depends on the scale at which the data are analysed.
Quite typically at a small scale, the ID is very large, as the data are affected by measurement errors. At large scale, the ID can also appear erroneously large, due to the curvature and the topology of the manifold containing the data. In this work, we introduce an automatic protocol to select the sweet spot, namely the correct range of scales in which the ID is meaningful and useful. This protocol is based on imposing that for distances smaller than the correct scale the density of the data is constant. In the presented framework, to estimate the density it is necessary to know the ID, therefore, this condition is imposed self-consistently. We illustrate the usefulness and robustness of this procedure to noise by benchmarks on artificial and real-world datasets.
\end{abstract}

\section{Introduction}
\label{sec:intro}
The study of high-dimensional data is a fundamental task in modern statistics, as many real-world datasets are often characterized by a large number of variables, a fact that may hinder their interpretation and analysis.
In most high-dimensional applied settings, a subset of the variables or non-linear combinations of them are often sufficient to describe the dataset.
The Intrinsic Dimension (ID) represents the minimum number of independent variables that allow to locally reproduce the data with minimal information loss.
This quantity is an integral part of several data science methods, from dimensionality reduction, where it can provide a meaningful target (\citealp{rozza2012novel,gliozzo2024intrinsic,van2009dimensionality}), to density estimation or clustering, where it can be used to improve the accuracy and robustness of existing algorithms (\citealp{rodriguez}).
Furthermore, the ID has become increasingly used as a stand-alone data-driven observable to quantify the complexity of datasets.
Usage of ID estimation in this sense can be found in different scientific disciplines such as in biology (\citealp{facco2019intrinsic}), genomics (\citealp{macocco}), quantum physics (\citealp{mendes2021intrinsic,mendes2021unsupervised}), computer science (\citealp{ansuini2019intrinsic,valeriani2024geometry}), image analysis (\citealp{pope2021intrinsic}).
The presence of (non-linear) correlation in data is still fostering the research of low-rank (tensor) decomposition methodologies, which can be employed to ease the tasks of estimating the density (\citealp{amiridi2022low,vandermeulen2021beyond}), finding meaningful latent-variable models (\citealp{anandkumar2014tensor}) and understanding the patterns and training dynamics of weights of deep neural networks (\citealp{jacot2023bottleneck,ledent2025generalization,balzano2025overview}), ultimately to make them more interpretable and efficient. The ID can play a role of setting a lower bound to the rank decomposition, allowing for more grounded and faithful low-rank decompositions.
A further context where the intrinsic dimension is recently arising is in the theoretical analysis of various statistical methods, in particular regarding convergence rates when data are assumed to have a low-dimensional structure, see e.g. \cite{kpotufe2011k}, \cite{nakada2020adaptive}, \cite{rosa2024posterior}, \cite{berenfeld2022estimating}, \cite{rosa2024nonparametric}, \cite{chakraborty2024statistical}, \cite{tang2024adaptive}.

Many different, but practically equivalent, ID definitions have appeared in the literature.
For example, it has been defined as the dimension of the support of the density which is generating the data. A comprehensive overview of these definitions can be found in (\citealp{campadelli}).
To estimate the ID, many approaches have been proposed such as projective (e.g.\ \citealp{fukunaga1971algorithm}), fractal (e.g.\ \citealp{grassberger1983measuring,erba2019intrinsic}), optimal transport (e.g.\ \citealp{block2022intrinsic}), and graph-based methods (e.g.\ \citealp{costa2005estimating,serra2017dimension}).
In this paper we focus on nearest-neighbour (NN) methods and we build on some very recently proposed NN ID estimators (\citealp{facco2017,levina2004maximum,denti,qiu2022intrinsic,macocco}), as they are statistically well-grounded and hence allow for uncertainty quantification through exact or asymptotic results.
Despite the favourable statistical foundations, all NN-based ID estimators possess a set of arbitrary free parameters that strongly control the scale at which one observes the data or, equivalently, the level of the dataset decimation (\citealp{facco2017,denti}).
This often leads to unclear choices for the right scale or the proper decimation level to consider, which can easily produce significantly biased estimates.
To address this issue, we propose an adaptive estimation methodology with provable theoretical guarantees. The methodology is formulated in a likelihood-based framework and further extended to its Bayesian counterpart. As we will show, our proposal brings three main advantages. 
First of all, it allows to self-consistently verify 
that the density of the data is locally constant. This is a crucial step when employing any NN-based methodology both in ID estimation as well as in density estimation.
The second advantage is related to the fact that many real datasets are affected by very noisy measurements. We will show that our proposal allows us to overcome this issue by automatically escaping the noisy neighbourhoods of the observations, providing better estimates with respect to all the other NN-based estimators. 
Finally, and most importantly, our methodology allows for the automatic selection of the scale at which the researcher should look at the data. 
This is done in a point-wise adaptive fashion, which avoids the need for an arbitrary choice of the scale which is, conversely, typically kept constant across the entire dataset. 
The optimal scale identification is a fundamental and often overlooked feature/passage. Real-world datasets can often be described by hierarchical structures characterized by a density varying from region to region. For such reasons, providing an ID estimate without referring to its scale can be misleading.

The remainder of the paper is organized as follows. 
In Section \ref{sec:meth}, 
we introduce our adaptive ID estimator.
In Section \ref{sec:num}, we illustrate our methodology on some simulated and real datasets, comparing its performance with non-adaptive procedures. 
Some concluding remarks are offered in Section \ref{sec:disc}.
The Appendix reports theoretical guarantees for the proposed methodology, additional simulation experiments and some extensions of the method.

\section{Methodology}
\label{sec:meth}
\subsection{Binomial Intrinsic Dimension Estimator (BIDE) }
We introduce our general modelling framework, which builds on \cite{macocco}, where the authors propose the Intrinsic Dimension estimator for Discrete Datasets (I3D), specifically tailored for estimating the ID in discrete spaces. Let $(\Omega, \mathcal{A}, P)$ be a probability space, and let $\mathcal{M}^D$ be a $D$-dimensional metric measure space. Let us consider $X_1, \dots, X_n$, independent copies of a random vector $X:\Omega \to \mathcal{M}^D$, sampled in a small neighbourhood of a $d$-dimensional subset of $\mathcal{M}^D$, where the size of the neighbourhood may depend on the scale of a random noisy perturbation of the samples, which can make such neighbourhoods appear effectively higher-dimensional.
Now, consider a measurable subset $A\subset \mathcal{M}^D$ with measure $\mu(A)$, and let us model the locations of the data points by a Poisson point process with intensity function $h: \mathcal{M}^D \to \mathbb{R}_+$ that is constant and equal to $\rho$ on the subset $A$. Thus, the number of data points in $A$ denoted by $k_A$ is a Poisson random variable with parameter $\rho \mu(A)$ and probability mass function (pmf) $p_{k_A}$ given by
\begin{equation*}
    p_{k_A}(x) = \frac{(\rho \mu(A))^{x}}{x!}\exp(-\rho \mu(A)).
\end{equation*}
Now, let $A\subset B \subset \mathcal{M}^D$ where $B$ is a further measurable subset and assume $h$ is constant and equal to $\rho$ also on $B$, we get the distribution of the number of points in $B\setminus A$, namely, $k_B-k_A\sim \mathrm{Poisson}(\rho \mu(B\setminus A))$ and the conditional distribution $k_A|k_B \sim \mathrm{Binomial} (k_B, p)$ with $p= \mu(A)/\mu(B)$, 
provided that the underlying Poisson process is locally homogeneous, i.e.\ homogeneous on the scale of the $B$ set.
Note that the intensity function can be regarded as the non-normalized measure underlying the data-generating process, therefore the local homogeneity of the Poisson process is often regarded as a constant density, assuming it exists.

At this stage, we can specify the nature of $\mathcal{M}^D$. We take $\mathcal{M}^D\equiv\mathbb{R}^D$ equipped with the Euclidean metric, $\mu$ is the Lebesgue measure on $\mathbb{R}^D$, and the data are sampled in a small neighbourhood of a $d$-dimensional $C^1$ manifold. Thus, at sufficiently small scales, Euclidean distances in $\mathbb{R}^D$ are well approximated by distances in the tangent space of the manifold, and volumes of small balls scale like those of a $d$-dimensional Euclidean ball.

Now, let $A$ and $B$ open balls centred at the same point on the manifold such that $A \subset B$. Then, for a sufficiently small outer ball $B$, $p \approx (t_A/t_B)^d = \tau^d$, where $t_A$ and $t_B$ are the radii of the two balls and $d$ corresponds to the ID.
Let $k_{A,i}$ and $k_{B,i}$ be, respectively, the number of points in the open balls $A$ and $B$ centred at unit $X_i=x_i$, where we do not count $x_i$ since we condition on it. We consider a likelihood formulation introducing
\begin{equation*}
    L(d)= \prod_{i=1}^n \binom{k_{B,i}}{k_{A,i}}  (\tau^d)^{k_{A,i}} (1-\tau^d)^{k_{B,i}-k_{A,i}},
\end{equation*}
which is maximized at 
\begin{equation}
\label{eq:bide}
    \widehat d = 
\frac{\log((\sum_{i=1}^n k_{A,i}) / (\sum_{i=1}^n k_{B,i}))}
     {\log\tau}.
\end{equation}
The estimator $\widehat d$ will be denoted as the \emph{Binomial ID Estimator} (BIDE).
Remark that to identify $d$, it is sufficient to fix just two tuning parameters among $t_A$, $t_B$ and $\tau$. In this paper, we will propose an adaptive data-driven choice for such tuning parameters.
It is worth mentioning that the estimator for Discrete Datasets (I3D) presented in \cite{macocco} is obtained setting $\mathcal{M}^D\equiv \mathbb{Z}^D$ equipped with the $L^1$ metric where volumes are computed using the counting measure and Ehrhart polynomials, thus, leading to an alternative non-closed expression of $\widehat d$. Moreover, it is also empirically shown by the authors that violating the independence assumption does not induce significant errors as long as $\frac{1}{n}\sum_{i=1}^n k_{A,i}\ll n$.

We observe that the ID estimation framework outlined above shares similarities with \cite{10.1145/1273496.1273530}, where the authors first obtain local ID estimates and then aggregate them by averaging or ``voting'' to get a global estimate.
Differently, we adopt a likelihood-based global approach, that notably leads to more robust global estimates; this concept is remarked also in \cite{MacKay}.

\subsection{Largest uniform neighbourhoods in nearest-neighbour methods}

Recall that we assumed that the data are sampled in a small neighbourhood of a $d$-dimensional $C^1$ submanifold of $\mathbb{R}^D$.
Thus, at sufficiently small scales, volumes of small balls scale like those of a $d$-dimensional Euclidean ball, 
which we use to model volumes of small balls.
Let $r_{i,j}$ be the distance between $i$ and its $j$-th NN with $r_{i,0}$ conventionally set equal to $0$, and denote by $B(x_i,r_{i,k})$ the ball centred at realization $X_i=x_i$ with radius $r_{i,k}$. The volume of the spherical shell between neighbours $j-1$ and $j$ of a point $i$ is given by
$$
v_{i,j}: = \mu(B(x_i,r_{i,j})\setminus B(x_i,r_{i,j-1})) =\Omega_d (r^d_{i,j}-r^d_{i,j-1}),
$$
where $\Omega_d = (2\Gamma(3/2))^d/\Gamma(d/2+1)$ is the volume of a $d$-dimensional unit ball with $\Gamma$ denoting the Euler's Gamma function.
Accordingly, the volume of the ball centred in $i$ with radius given by its $k$-th NN is
$$
V_{i,k}:= \mu (B(x_i,r_{i,k}))= \sum^k_{j=1} v_{i,j} = \Omega_d \,r^d_{i,k}.
$$
If we consider the intensity constant and equal to $\rho_i$ within a point $i$ and its $k$-th NN at distance $r_{i,k}$, then spherical shell volumes $v_{i,j}$ with $j\in\{1,\dots,k\}$ are the realizations of $k$ independent and identically distributed (iid) exponential random variables with rate $\rho_i$.
Therefore, we can write a joint probability density function for $k$ iid exponential random variables $f(v_{i,1},\dots,v_{i,k}) =\prod_{j=1}^{k} \rho_i e^{-\rho_i \, v_{i,j}} $.
Thus, looking at the joint density, evaluated at the observed spherical shell volumes $v_{i,1},\dots, v_{i,k}$, as a function of the parameter $\rho_i$, and by taking its logarithm we obtain the log-likelihood
\begin{equation*}
L_{i,k}(\rho_i): =  \sum_{j=1}^{k} \log (\rho_i e^{-\rho_i \, v_{i,j}}) =  k \log \rho_i - \rho_i V_{i,k},
\end{equation*}
which is maximized at $\widehat{\rho}_{i} = k/V_{i,k}$, matching the $k$-NN local non-normalized density estimator. Its standard deviation, given by $ \widehat{\rho}_i/\sqrt{k}$, is obtained as the inverse square root of the Fisher information.
It is worth noticing the usual bias-variance trade-off: the estimator's standard deviation decreases as $k$ increases, and, on the other hand, the constant density approximation, on which the exponentiality assumption relies, becomes less solid since the density in a ball of radius $r_{i,k}$ might become non-constant as $k$ increases. 
Therefore, as we have observed, the $k$-NN density estimator is not immune to the usual bias-variance trade-off problem. However, since the influence of $k$ on the variance can be somewhat decoupled from its impact on the bias, there are margins of improvement.
Indeed, to achieve low variance the estimator typically requires a high number of observations or the selection of a large $k$, and for large $k$ a systematic bias is introduced primarily for the statistical units for which the region enclosing $k$ NNs undergoes rapid density variations.
In essence, the selected bandwidth for the $k$-NN estimator should be small in regions where the density of data points changes rapidly but it can be larger in areas where the density undergoes slower variations.

Following this reasoning, \cite{rodriguez} developed a quantitative algorithm to find the largest $k$, specific to each data point, that enforces an approximately constant density within a certain confidence level, thus facilitating and automatizing more accurate density estimation in regions where the density exhibits gradual variations. We formalize the procedure to find the optimal $k$ in a statistical hypothesis testing framework.
For each point $i$ we consider a growing number of neighbours and iteratively compare two different likelihood models.
The first model denoted as $M1$, assumes that the densities of the Poisson process at point $i$ and at its $(k+1)$-th NN are different, say $\rho_i$ and $\rho '_i$.
In this case, the log-likelihood evaluated at its maximisers is
\begin{equation}
\label{eq:M1}
    L^{M1}_{i,k}(\widehat \rho_i, \widehat \rho'_i)= \underset{\rho_i,\rho'_i>0}{\max} (L_{i,k}(\rho_i)+L_{k+1,k}(\rho'_i))=k \log \Big(\frac{k^2}{V_{i,k}V_{k+1,k}}\Big)-2k
\end{equation}
The second model denoted as $M2$, on the other hand, assumes no density variation between $i$ and its $(k+1)$-th NN, leading a maximised log-likelihood given by
\begin{equation}
\label{eq:M2}
    L^{M2}_{i,k}(\widehat \rho_i, \widehat \rho_i)
    = \underset{\rho_i>0}{\max} (L_{i,k}(\rho_i)+L_{k+1,k}(\rho_i))=2k \log \Big(\frac{2k}{V_{i,k}+V_{k+1,k}}\Big)-2k.
\end{equation}
Fundamentally, the volumes $V_{i,k}$ and $V_{k+1,k}$ in the above equations can be taken to be the volumes of balls on the intrinsic manifold by appropriately using $d$ instead of the embedding dimension $D$ for their computation.
The two maximised log-likelihoods are then compared using a likelihood-ratio test statistic. As long as they are deemed statistically compatible within certain confidence, the putative optimal value for $k$ is increased by one and a new test is carried out; when the likelihood of model ${L}_{i,k}^{M1}$ utilises its extra degree of freedom ($\rho'_i$) in the maximization to become significantly larger than ${L}_{i,k}^{M2}$ it means that the density variation becomes significant and, thus, the constant density approximation does not hold for the tested $k+1$, meaning that $k$ is the optimal neighbourhood size for point $i$.
More precisely, by Wilks' theorem (\citealp{wilks1938large}) the likelihood-ratio test statistic
\begin{align*}
\begin{split}
        D_{i,k}&= -2(L^{M2}_{i,k}(\widehat \rho_i , \widehat \rho_i) - L^{M1}_{i,k}(\widehat \rho_i, \widehat \rho'_i))\\
        &= -2k\left(\log V_{i,k}+\log V_{k+1,k} - 2 \log (V_{i,k}+V_{k+1,k})+ \log 4 \right)
\end{split}
\end{align*}
converges in distribution to a $\chi^2_1$ random variable, and it can be used to test $H_0: \rho_i = \rho'_i$ against $H_1: \rho_i \neq \rho'_i$. Since the test can be performed for every $k$ it follows that for each point $i$ we can select $k^*_i$ as 
\begin{equation}
\label{eq:kstar}
    k^*_i =  k_{max} \wedge \min \{k : D_{i,k}\geq D_\text{thr} \} 
\end{equation}
where $D_\text{thr}$ is the rejection threshold and can be selected as the $(1-\alpha)$-order quantile of the $\chi_1^2$ distribution.
The choice of $\alpha$ is strongly linked to the usual bias-variance trade-off or the signal-noise disentanglement; in fact, the amount of information that should be regarded as noise is ultimately a choice of the field experts, and the proposed modelling framework allows a clear and transparent control on such choice. In this perspective, we also remark that $\alpha$ acts as a regularization parameter rather than a globally calibrated significance level.

Thanks to this framework, the neighbourhood sizes $k_1^*,\dots,k_n^*$, are adaptively optimised for each data point, and can be plugged in all the results of the canonical $k$-NN to give birth to an improved $k^*$-NN estimator.
It is worth noticing that computational complexity of the method is linear in the sample size, as we fix the maximum extension of the neighbourhood to a large yet finite value $k_{max}$.

\subsection{Adaptive Binomial Intrinsic Dimension Estimator (ABIDE)}
We are now ready to present the new algorithm for ID estimation which automatically uses optimal neighbourhoods. 
The algorithm can be seen as an extension and improvement of BIDE. In particular, the idea revolves around the concept of exploiting the largest possible neighbourhoods for the computation of the ID while, at the same time, satisfying the locally constant density hypothesis encapsulated in the Poisson process modelling framework at the basis of the estimator.
Concretely, this means computing the ID with the binomial estimator fixing the size of the neighbourhood not to a constant value $k$ but to the adaptively chosen values $k^*_i$ found by the optimisation procedure described above.
In this manner we can select, separately for each data point $X_i$, the largest possible neighbourhood where the hypothesis of constant density is not rejected.
The distance of the $k_i^*$-th neighbour from point $i$ is then $r_{i,k_i^*}$.

There is an important caveat that makes the methodology much more challenging. 
The algorithm that returns the optimal $k_i^*$ requires the knowledge of the ID to compute volumes of balls; see \eqref{eq:M1} and~\eqref{eq:M2}. 
To address this key issue we propose the following iterative procedure.
First, an initial value of ID has to be found using a standard method like the BIDE or the Two Nearest Neighbours (2NN) estimator \cite{facco2017}.
Since the 2NN is well-established, it does not have tuning parameters and it focuses on very small neighbourhoods of order 2, we adopt it as the starting value for $d$.
Moreover, it is interesting to note that the 2NN in its maximum likelihood formulation is equivalent to BIDE under certain conditions. To see this, we recall that the 2NN estimator in its basic maximum likelihood formulation \citep{denti} is given by
$$\widehat d_{2NN}= \frac{n}{\sum_{i=1}^n \log (r_{i,2}/r_{i,1}) }.$$ 
Therefore, setting $\widehat d (\tau) = \widehat d_{2NN}$ and solving for $\tau$, shows that BIDE and 2NN are equivalent when $\tau = \exp( \log((\sum_{i=1}^n k_{A,i})/ (\sum_{i=1}^n k_{B,i} )) n^{-1} \sum_{i=1}^n \log ( r_{i,2}/r_{i,1}))$. Moreover, note that if the count ratio $(\sum_{i=1}^n k_{A,i})/ (\sum_{i=1}^n k_{B,i})$ is close to  $1/e$, then $\tau$ is close to the geometric mean of $\{r_{i,1}/r_{i,2}\}_{i=1,\dots,n}$.

Now, proceeding with the description of our iterative procedure, once we have set a starting value for $d$, the optimal set of $k^*_1,\dots,k^*_n$ is then computed according to \eqref{eq:kstar},
and used as input for BIDE which gives a new ID estimate. 
Such ID, typically different from the one obtained with the 2NN, is used to compute another set of $k^*_1,\dots,k^*_n$, that will become the input for an updated ID estimate. 
The procedure continues until convergence of the ID estimate.
More precisely, at each iteration, once $k_1^*,\dots,k_n^*$ have been computed, $d$ is estimated using BIDE as in \eqref{eq:bide}:
it suffices to set $k_{B,i}= k_i^*-1=:k_{B,i}^*$ and $t_{B,i}(k_i^*)=r_{i,k_i^*}$ as the radius of the open ball containing $k_{B,i}^*+1$ points. 
Successively, we compute $k_{A,i}^*$ as the number of points contained in the open ball of radius $t_{A,i}(k_i^*)=t_{B,i}(k_i^*)\tau$ without counting point $i$, and we set $k_{A,i}=k_{A,i}^*$. 
Furthermore, we can devise an optimal procedure to select the ratio $\tau$ of the radii of the two open balls.
As shown in the supplementary material of \cite{macocco}, minimizing the asymptotic variance of the estimator in \eqref{eq:bide} leads to the optimal choice $\tau \approx c_*^{1/d}$ where $c_*=0.2032$.
Since the optimal $\tau$ depends on $d$, the true ID, we can naturally include the computation of $\tau$ in the iterative procedure by using the latest estimate of the ID at every step in place of its true value.
We refer to this procedure as \emph{Adaptive Binomial ID Estimator} (ABIDE) and denote the resulting estimator with $d^*$.

\begin{figure}[!htb]
    \centering 
    \includegraphics[width=1.0\linewidth]{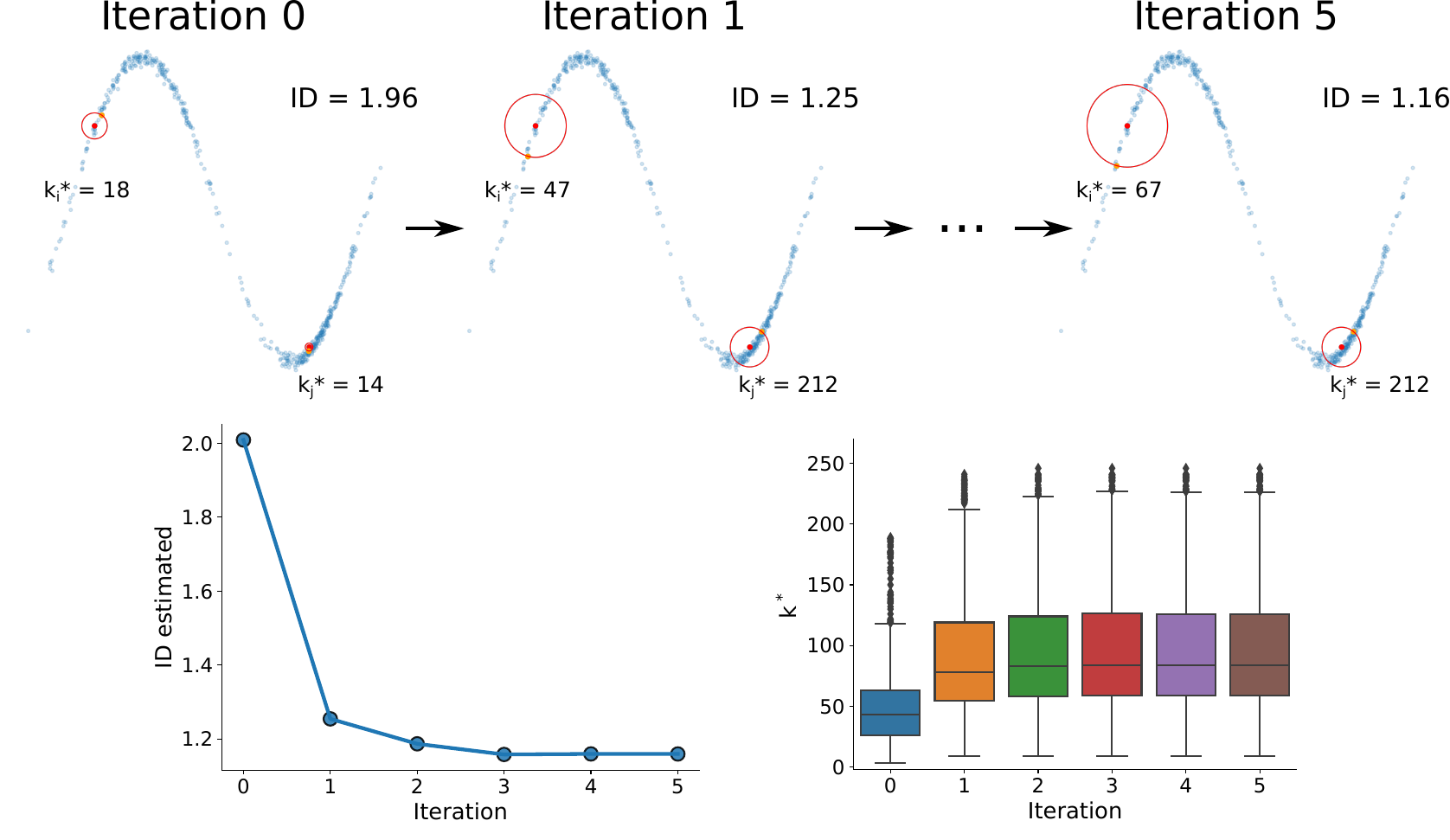}
    \caption{The ABIDE algorithm, thanks to its iterative nature, progressively adjusts the estimates of the ID and the size of (approximately constant density) neighbourhoods, allowing it to escape the noise scale and to find the true ID of the underlying manifold. 
    The first row shows how, throughout the iterations, the size of the neighbourhoods for two selected points $i$ and $j$ grows as the ID estimates lower from 2 down to 1.
    In the second row, we report the evolution of the ID estimate and the boxplots representing the subsequent distributions of the $k^*$.
    }
    \label{fig:sketch}

\end{figure}
In Figure~\ref{fig:sketch} we report a sketch of our algorithm on a toy dataset created in the following way: 
We sample 1000 points such that $X_1,\dots,X_{500}\sim\mathcal{N}(\pi/2,1)$, $X_{501},\dots,X_{1000}\sim\mathcal{N}(5\pi/3,0.5^2)$ and $Y_i=\sin(X_i)+\epsilon_i$, where $\epsilon_i\sim\mathcal{N}(0,0.025^2)$. The standard deviation of the noise sets the scales below which the dataset appears to be 2-dimensional.
The bottom left panel shows how the one-dimensional nature of the data emerges throughout the successive iterations of ABIDE. Simultaneously, the neighbourhoods where the density can be considered approximately uniform get larger and larger, as shown in the bottom right panel.
The reason is that, as long as the estimated ID is approximately 2 (which is the value obtained using 2NN at iteration 0) the manifold is expected to be 2-dimensional, and thus the uniform density test is performed onto 2-dimensional balls (disks).
Accordingly, the size of the uniform neighbourhoods is limited to scales comparable with the typical magnitude of the noise along the $Y$ coordinates, beyond which the density of data points belonging to a (supposedly) 2-dimensional manifold drops to 0.
However, the extension of such neighbourhoods is large enough so that, at iteration 1, the estimated ID using BIDE in combination with the $k_i^*$ is of order 1.3.
Consequently, as the estimated ID gets closer to one, the uniform density test is performed using the correct dimension, allowing to obtain larger values for $k^*$, since the neighbours' size can extend along the proper manifold, well beyond the scale of the noise.

This example makes the interplay between the estimated ID and the size of uniform neighbours evident: it is exactly because of larger neighbourhoods that it is possible to escape the scale at which the noise is relevant and, thus, uncover the true dimensionality of the data. At the same time, the proper size of the neighbourhoods is found using the right ID.
From this perspective, it is evident how the process is intrinsically iterative.

As we will show in the Results section, in most cases few iterations are needed to obtain a stabilization of the ID estimate and of the associated $k^*_i$ values. 
Intuitively, this can be understood by the fact that the ratio $(\sum_{i=1}^nk^*_{A,i})/(\sum_{i=1}^nk^*_{B,i})$ is not strongly affected by a possibly inaccurate initial estimate of the ID, since both numerator and denominator grow proportionally if the data generating density is sufficiently regular.
In the Appendix, we provide a formal justification for the fast numerical convergence of Algorithm \ref{alg:adaptive}. 
Specifically, we show that Algorithm \ref{alg:adaptive} is a fixed point iteration and use some theoretical results regarding the large sample behaviour of the associated contraction mapping to show that it terminates.
In practice, we observe that 5 iterations are often enough to achieve a stabilization of the ID estimate with a tolerance $\delta$ of order $10^{-4}$, and most of the improvement is obtained in the first 2 iterations. In practice, we observe that also $k_1^*,...,k_n^*$ stabilise when the ID estimate stabilizes.
Thus, the suggested iterative procedure is highly scalable and brings substantial improvement over the starting estimate obtained using the 2NN with little additional computational cost. 
In the numerical experiments reported below, instead of selecting a value for $\delta$, we choose to fix the number of iterations to 5.
Note that the time complexity degrades to $O(D n^2)$ for the initial nearest-neighbour calculation when the ambient dimension $D$ is moderately large (around $20$ or larger), and is $O(n)$ for both the intrinsic dimension (ID) and $k^*$ computations (see also \citealp{dadapy_final}). In Section \ref{sec:bonferroni}, we report the execution times for the first two numerical experiments.
ABIDE is formally summarised in Algorithm~\ref{alg:adaptive} and, in the Appendix, we provide large-sample guarantees.
\begin{algorithm}[!htb]
    \caption{Adaptive-BIDE workflow}
    \label{alg:adaptive}
    \begin{algorithmic}[1]
        \State $d_\text{current}\gets$ ID obtained from the 2NN estimator
        \State $d_\text{next}\gets 0$
        \For { $it <$ max\_iter }
        \State $\tau=c_*^{1/d_\text{current}}$ 
         \For { $i <n$ }
            \State compute $k^*_i$ (using $d_\text{current}$) and set $k_{B,i}^*=k_{i}^*-1$
            \State $t_{B,i}(k_i^*)=r_{ik_i^*}$ 
            \State $t_{A,i}(k_i^*)= \tau \, t_{B,i}(k_i^*)$
            \State $k_{A,i}^*=\sum_{j=1}^{n-1}\mathbf{1}\{t_{A,i}(k_i^*)-r_{i,j}>0\}$
                \EndFor
            \State $d_\text{next}= \frac{\log((\sum_{i=1}^n{k_{A,i}^*}) / (\sum_{i=1}^n{k_{B,i}^*}))}{\log\tau}$
            \If {$|d_\text{current}-d_\text{next}|<\delta$}
                \State \textbf{break}
            \EndIf
            \State $d_\text{current}= d_\text{next}$
        \EndFor
        \State $d^*=d_\text{next}$
         \For { $i <$ n }
          \State compute $k^*_i$ (using $d^*$)
            \EndFor
        \State \textbf{return} $d^*,\;k^*_1,\dots,k^*_n$
    \end{algorithmic}
\end{algorithm}
Normal approximation of ABIDE allows for a simple uncertainty quantification of the estimates bypassing the need to use computationally intensive resampling procedures such as bootstrap or jackknife.
Indeed, the approximate confidence interval reads 
\begin{equation*}
    \Big[  d^* - \frac{z_{1-\beta/2}}{\sqrt{n I(d^*)}}  , d^* + \frac{z_{1-\beta/2}}{\sqrt{n I(d^*)}}\Big]
\end{equation*}
where $z_{1-\beta/2}$ is the quantile of order $1-\beta/2$ of the standard Gaussian distribution and $I(d^*)$ is the observed Fisher information evaluated at $d^*$ which is given by
\begin{equation}
\label{eq:fisher}
    I(d^*)=\frac{(\log\tau)^2 \tau^{d^*}n^{-1} \sum_{i=1}^n k_{B,i}^*}{1-\tau^{d^*}}.
\end{equation}
Note that the data-driven optimal choice $\tau= c_*^{1/d}$ implies that the approximate variance is proportional to $d^2/\sum_{i=1}^n k_{B,i}^*$.

To assess the reliability of the estimated ID and the selected neighbourhoods, one can exploit a similar model validation framework proposed in \cite{macocco}. 
In particular, at each iteration and for the values obtained at termination, the set of $k^*_i$ and the estimated ID are used to build a theoretical distribution with pmf $\widetilde p_{ k_A^*}$, obtained as a mixture of binomial distributions with pmf $p_{k_A^*|k_B^*}$ and mixing pmf $p_{k_B^*}$ given by the empirical pmf of $k_{B,1}^*,\dots, k_{B,n}^*$. Thus, letting $p_{k_A^*}$ be the empirical pmf of the observed $k_{A,1}^*,\dots,k_{A,n}^*$, the theoretical mixture distribution with pmf
\begin{equation*}
   \widetilde p_{ k_A^*}(x) = \sum_{y \geq 0} p_{k_B^*}(y) p_{k_A^*|k_B^*}(x|y)
\end{equation*}
is compared to $p_{k_A^*}$ through a two-sample test for equality in distribution, which is performed between an arbitrarily large artificial sample drawn from  $\widetilde p_{ k_A^*}$ and the observed $k_{A,1}^*,\dots,k_{A,n}^*$.
Differently from \cite{macocco}, who adopt a standard Kolmogorov-Smirnov test, in our analysis we choose the Epps-Singleton test (\citealp{epps1986omnibus}), which is based on the empirical characteristic function and is also valid for discrete distributions.
In the goodness-of-fit analysis, a large p-value of the test indicates that the modelling assumptions are more likely to be met and therefore the ID estimate is reliable. Note that such p-value should be read as a relative measure of goodness-of-fit useful for comparison, as it is unlikely that the binomial mixture law is exactly satisfied, and it is well-known that goodness-of-fit tests can be highly sensitive to small departures from the model when the sample size is large.

\section{Numerical experiments}
\label{sec:num}
In this section, we illustrate the behaviour of ABIDE and comment on its performance, in particular against non-iterative algorithms. 
Assuming that in the applied analyses we want to achieve robustness to low or moderate noise magnitudes, we opt for a conservative approach regarding the null hypothesis of constant density, enforcing a moderately high sensitivity to density variations.
This means that it must be relatively easy to reject $H_0$, and thus accept the alternative hypothesis that states that the density is effectively changing. 
To this aim, we fix $k_{max}=350$ (which is actually never reached in our experiments), set $\alpha =0.01$ and, accordingly, take $D_\text{thr}=6.635$. We remark that we do not claim that such choice leads to corresponding global empirical calibration, as this is not the main scope of the proposed procedure, even though the $k^*$ selection is framed as an hypothesis test.
This is a cautionary choice, as it ends up selecting small uniform neighbourhoods.
Smaller values of $\alpha$ would result in the selection of larger neighbourhoods and higher robustness of the estimator to noise, at the cost of potentially introducing bias in non-noisy scenarios.
Further numerical experiments show the effect of this choice in the Appendix, where we also consider multiple testing based selection.
We show that the results are fairly robust to such selection, however, the relevance of this trade-off between noise robustness and bias will be the object of future investigations.

\subsection{Noisy Gaussian distribution}
\label{sec:noisy_gaussian}
The first simulated scenario consists of a $d$-dimensional multivariate standard Gaussian random variables with independent components embedded in a $D$-dimensional space. 
The embedding is obtained by adding noise to all the $D$ coordinates. More precisely, let us consider the random vector $(S_1,\dots, S_d)\sim \mathcal{N}(0_d, \Sigma)$ with $\Sigma = \sigma_S^2 I_d $, and consider $\varepsilon_{1},\dots,\varepsilon_D$ independent copies of $\epsilon \sim \mathcal{N}(0,\sigma_\varepsilon^2)$, where $\sigma_S$ and $\sigma_\varepsilon$ represent respectively the signal and the noise scale, and, thus, $\sigma_\varepsilon/\sigma_S$ is the noise-to-signal ratio.
In the notation used above, we let $X=(S_1+\varepsilon_1, S_2+\varepsilon_2, \dots, S_d+\varepsilon_d, \varepsilon_{d+1},\dots,\varepsilon_{D})$ and consider independent copies $X_1,\dots,X_n$ representing the rows of the data matrix.
In the simulation study, we compare ABIDE with the 2NN estimator which is also adopted as the starting point of the ABIDE iteration. 
To assess the variability of the estimators we generate 500 independent Monte Carlo replicas 
with sample size $n=5,000$, and, compute the Monte Carlo empirical quantiles to construct the Monte Carlo 99\% confidence interval (execution times are reported in Section~\ref{sec:bonferroni}). 
We consider $\sigma_S=1$ and let $\sigma_\varepsilon$ range in $[10^{-4},10^{-1}]$.
For all the experiments we set $D=100$ and $d \in \{2,5\}$.
This is a misspecified setting because the ID that we aim to retrieve is the one that we would obtain if the data were non-noisy. 
More precisely, misspecification means that there is a mismatch between the $d$ we aim to estimate and the usual ID definitions that lead to the construction of estimators. 
Indeed, it is not possible to project the highly noisy data into a $d$-dimensional space without loss of information. Moreover, the dimension of the support of the probability distribution generating the data turns out to be larger than $d$, and this is the typical case when dealing with real data.
Nevertheless, Figure \ref{fig:simul1_noise} shows that the proposed estimator behaves well in terms of bias and variance, and, compared to the 2NN, improves protection against noisy misspecified scenarios and achieves smaller uncertainty (shaded areas) on the estimates.

 \begin{figure}[!ht]
  \centering
  \includegraphics[width=1.0\linewidth]{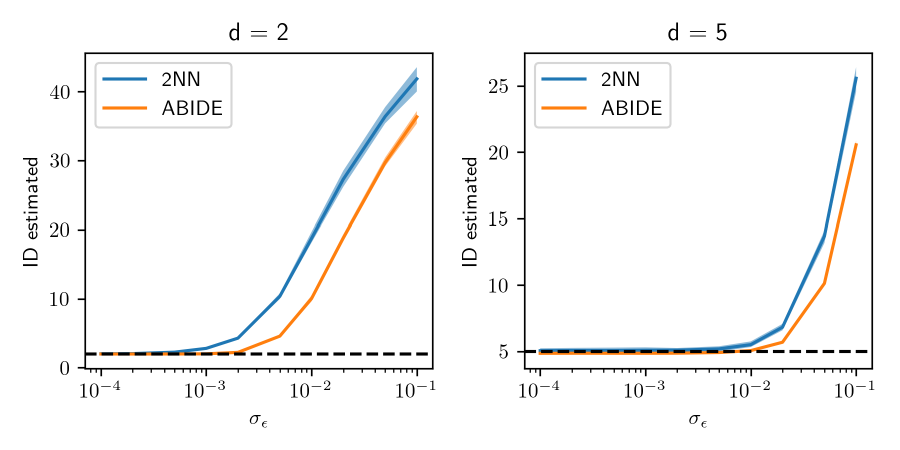}
  \caption{
  Estimated ID as a function of embedding noise for two Gaussian datasets of different $d$ and with $n=5,000$ points. 
  The shaded area is the Monte Carlo 99\% confidence interval and the black dotted line is the true $d$.
  }
  \label{fig:simul1_noise}
\end{figure}

\subsection{Noisy-curved-inhomogeneous 2-dimensional distribution}
\label{sec:moebius}
The second dataset we analyse is adapted from~\cite{d2021automatic}. Its original form, which we reported in Panel A of Figure \ref{fig:moebius1}, shows a uniform background with eight Gaussian distributions with different means and variances. As it is, such a dataset can be analysed with canonical tools, as any estimator would find an ID close to 2. 
To make the experiment more interesting, we first embed such points on a 3-dimensional M\"obius strip, which is then placed within a 20-dimensional space by adding a 20-dimensional Gaussian noise of scale $\sigma_\varepsilon=10^{-3}$.

\begin{figure}[!htb]
    \centering 
    \includegraphics[width=1.0\linewidth]{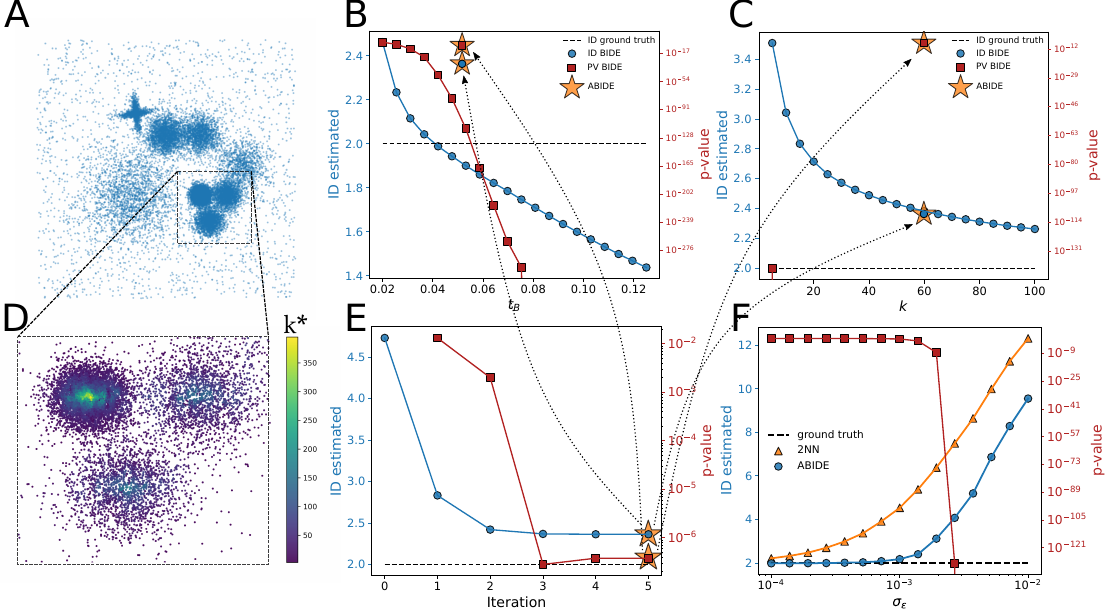}
    \caption{Panel A: the original 2-dimensional distribution of the 20,000 artificial data points under examination (taken from~\citealp{d2021automatic}). 
    The description of the embedding procedure to obtain the actual dataset is described in the main text. 
    Panel B: ID estimation using the standard non-adaptive Binomial Intrinsic Dimension Estimator (BIDE) at fixed radius or scale $t_B$ (blue circles, left y-axis) with associated p-values obtained through model validation (red squares, right y-axis in log-scale). The isolated starred points represent, respectively, the ID estimate (blue) and the p-value (red) obtained with the ABIDE estimator. 
    As a reference, the latter is placed in the average (over all data points) of the distance of the $k_i^*$ neighbour: $\frac{1}{n}\sum_{i=1}^n{t_{B,i}(k^*_i)}$.
    The dashed line (reported also in other panels) is the theoretical ID=2 of the original dataset before adding noise.
    Details and considerations in the main text.
    Panel C: ID estimation and p-values using BIDE at fixed neighbourhood size $k$ (on the x-axis). Also in this case ABIDE results are reported as the starred values, in correspondence of $\frac{1}{n}\sum_{i=1}^n{k^*_i}$.
    Panel D: 
    points are coloured according to their $k^*_i$ value to visualize the concept of adaptive neighbourhood. 
    Panel E: the evolution of the ID (blue) and the associated p-values (red) from model validation at the successive iteration of ABIDE. The final values at termination are starred and reported also in the BIDE plots (panels B and C) for comparison.
    Panel F: 2NN vs ABIDE and p-values as a function of the noise scale. In Section~\ref{sec:comparison}, we report the estimates obtained with the other most commonly used NN-based ID estimators.
    }
    \label{fig:moebius1}
\end{figure}

As a first test, we try to understand whether it is possible to observe a plateau at ID=2 as prescribed by conventional estimators by tuning the fixed radius, $t_B$, in BIDE. 
To check the goodness of the estimate we also perform the model goodness-of-fit test through the Epps-Singleton test and look at the obtained p-values.
The results are reported in Panel B of Figure \ref{fig:moebius1} and we cannot find the plateau around the ground truth ID value of 2. 
Instead, the estimated ID keeps decreasing as $t_B$ increases.
However, the presence of non-negligible noise, still relevant at such small scales, does not allow us to find the ID of the manifold within such neighbourhoods, as those are too small to overcome the noise.
Conversely, as soon as one measures the ID at larger scales by increasing $t_B$ to escape the noise, many of the selected neighbourhoods undergo high-density variations.
Accordingly, the p-values start declining together with the ID. As a consequence, if one uses the goodness-of-fit of the model as a criterion for choosing the radius $t_B$, one would be forced to conclude that the ID of this data set is significantly larger than the ground truth value.

To try to account for density variations, we fixed the neighbourhood size and computed the ID using BIDE with $k$ fixed. 
Also in this case the results, reported in Panel C of Figure \ref{fig:moebius1}, are unsatisfactory. 
Even if the ID appears to slowly converge to a value close to 2 as $k$ grows, the relative p-values are extremely low, which would lead to consider the ID obtained with a large $k$ unreliable. Density variations are too strong and this approach does not allow us to make reliable predictions, even if the ID trend is going in the proper direction.

Next, we apply the ABIDE algorithm (execution times are reported in Section~\ref{sec:bonferroni}).
As described above, the algorithm does not require fixing either $t_B$ or $k$, as those values are chosen automatically. 
The result of the ABIDE estimator is given by the ID (and its p-value) and the $k^*$ distribution (and eventually $t_B(k^*)$) at termination (starred values in Panel E of Figure \ref{fig:moebius1}). 
Therefore we do not have to change any parameter explicitly to find patterns (like plateaus) in the ID path.
The obtained ID and p-values are reported as starred isolated points in Panels B and C of Figure \ref{fig:moebius1} in correspondence of, respectively, $\frac{1}{n}\sum_{i=1}^n{t_{B,i}(k^*_i)}$ and $\frac{1}{n}\sum_{i=1}^n{k^*_i}$.
Quantitatively speaking, we obtained ID $\approx 2.35$, close to the true value of 2, and a p-value $\approx 5\times 10^{-7}$, which is many orders of magnitudes larger than the ones obtained at the same (averaged) scale or at fixed (for every point) neighbourhood size. This is a strong indication that choosing the neighbourhood size in a point-wise adaptive manner is the key approach if we want to obtain a statistically reliable description of the data distribution. 

We now analyse in more detail how the ABIDE results are obtained. 
Panel E of Figure \ref{fig:moebius1} shows the convergence of the ID (blue) to a value slightly above 2 within three iterations (the ID at iteration 0 was obtained employing the 2NN estimator). 
The corresponding p-values (red) are always of order $10^{-7}$ or $10^{-6}$, meaning that thanks to the adaptive neighbourhood implementation, we are properly exploring regions with a fairly constant density around each data point, resulting in a proper binomiality of the random variables $\{k_{A,i}^*|k_{B,i}^*\}_{i=1,\dots,n}$.
The values at convergence are then considered the results of the ABIDE estimator. These have been starred and reported, allowing for a direct comparison, in Panels B and C.

In the Appendix, we also report the distributions of $k_i^*$ and the related radii $t_{B,i}(k^*)$ obtained with 2NN and ABIDE. As occurred in the illustrative example of Figure \ref{fig:sketch}, 
the lowering of the ID is paired with an enlargement of neighbourhoods' size.
To visualize the concept of adaptive neighbourhood, in Panel D
we coloured the points of the dataset according to their value of $k^*$ obtained with ABIDE (for visualization reasons we used the original 2-dimensional dataset).
As one can see, points close to the border of the distribution present a small $k^*_i$, as the density sharply changes within a small number of neighbours. 
Conversely, points at the centre of a density peak show a much larger value of $k^*_i$, as the density varies in a sufficiently smoother way to be considered approximately constant on such a scale.

As a last experiment, we consider the same dataset for different noise intensities. In particular, we use a noise of $10^{-4}\le\sigma_\varepsilon\le 10^{-2}$ to embed the original 2-dimensional data points. 
We expect that the stronger the noise the harder it is to grasp the true dimensionality of the hidden manifold. 
A high noise intensity necessarily implies an overestimate of the manifold dimension, as more orthogonal coordinates are required to represent the data. 
Panel F of Figure \ref{fig:moebius1} reports the ID estimates obtained by 2NN and ABIDE (and its associated p-values) as a function of the noise intensity.
We observe how ABIDE is capable of recovering the ID of the data manifold much more robustly than the 2NN method, being able to tolerate a higher level of noise remaining much closer to the ground-truth ID value.
Indeed, as far as $\sigma_\varepsilon $ is approximately smaller than $10^{-3}$, ABIDE correctly finds ID estimates very close to 2 and, correspondingly, p-values of order $10^{-1}$. 
Differently, when $\sigma_\epsilon$ is approximately larger than $ 10^{-3}$, the noise becomes so intertwined with the manifold structure that not even ABIDE is capable of getting reasonable estimates of ID. 
Accordingly, the associated p-values drop dramatically, signaling that the estimates are no longer quantitatively reliable in such a regime. 
In Section~\ref{sec:comparison}, we report the results obtained for other commonly employed NN-based ID estimators, where it is shown that all of them have worse performance than ABIDE.

Already within this artificial example, we can identify the many advantages offered by ABIDE. Thanks to its iterative formulation, it can sequentially adjust the ID and the neighbourhood sizes to point-wise identify the largest region where the density is approximately constant and enforce the modelling framework behind the estimator.
At the same time, such a procedure allows us to surpass the barrier created by noise with higher capability than other methods and identify a proper adaptive scale at which the ID should be computed.

\subsection{Applications to image data and molecular dynamics trajectories}
In this section, we present the results of the ABIDE estimator on three real-world datasets with the following specifications.
The first one is the so-called OptDigits (\citealp{misc_optical_recognition_of_handwritten_digits_80}), a collection of data matrices of dimension $n=3,823$ and $D=64$, representing small-sized hand-written digits.
For the second application, we remain in the realm of hand-written digits with the paradigmatic MNIST dataset (\citealp{lecun2010mnist}), with size $n=70,000$ and $D=784$. Being an order of magnitude larger than the previous dataset concerning both the number of entries and features, it represents a very different challenge despite the similar nature of what is encoded in the data.
As a last example, we analyse a biomolecular trajectory. We use 
$n=3,758$ uncorrelated frames from a replica-exchange molecular dynamics simulation (\citealp{sugita1999replica}) at a temperature of 340 K of the CLN025, a small peptide made up of 10 residues that folds into a beta-hairpin (\citealp{honda_2004}). 
Each frame of the trajectory, initially computed in (\citealp{carli2021statistically}) and corresponding to a single data point, can be represented in many different ways.
We decided to work with the compact but still meaningful representation given by the dihedrals angles (\citealp{bonomi2009plumed,cossio2011similarity}). 
With this description, the dimensionality of the dataset is $D=32$. 
Since we are dealing with angles, the periodicity of $2\pi$ has to be taken into account when computing the distances between data points.

Figures~\ref{fig:opt_mnist} and ~\ref{fig:cln} show the analyses carried out similarly to what we did in the previous section on artificial data. 
In particular, for each one of these figures, the first row shows some samples extracted from the datasets. 
In the second row, from left to right, we report: ID (blue) and p-values (red) estimated with ABIDE at successive iterations; BIDE ID scaling at fixed $t_B$ with ABIDE results superimposed (as a starred blue bullet for the ID and a starred red square for the p-value); BIDE ID scaling at fixed $k$ with ABIDE results superimposed (as starred bullets); the distribution of $k^*$ computed using the 2NN estimate and the ABIDE estimate.

\begin{figure}[!htb]
    \centering
    \includegraphics[width=1.0\linewidth]{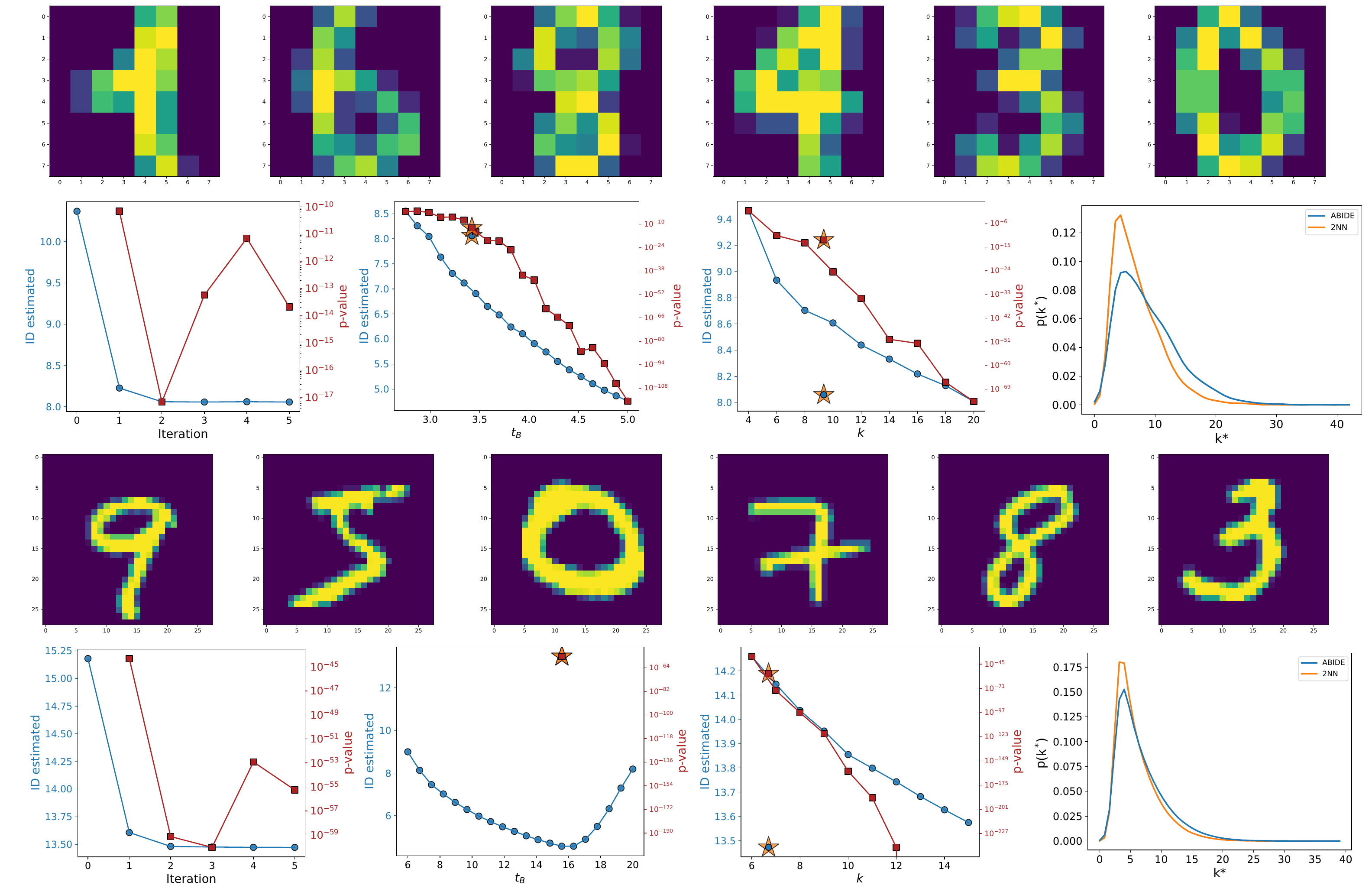}
    \caption{In the two upper rows, we report ABIDE performance on the OptDigits dataset made up of $n=3,823$ elements with embedding dimension $D=64$, together with a sample of 6 data points (first row). 
    The first panel shows the evolution of the ID (blue), and of the p-value (red), with successive ABIDE iterations. In particular, one can appreciate how the ID estimated with ABIDE stabilizes within 3 iterations and is significantly smaller than the 2NN estimate (from 10.43 to 8.05). The second and third panels show the scaling of the ID obtained with BIDE by fixing the radius $t_B$ or by fixing the number of neighbours $k$. 
    Since no plateau is present in the scaling curves, without ABIDE we would not have a proper criterion to recognize a meaningful ID.
    Conversely, ABIDE - shown as starred points centred, respectively, in $\frac{1}{n}\sum_{i=1}^n{t_{B,i}(k^*_i)}\approx 3.5 $ and $\frac{1}{n}\sum_{i=1}^n{k^*_i} \approx 9$ - allows us to uniquely identify the point-dependent scale as the largest neighbourhood size $k^*_i$ where the local density is approximately constant.
    The last panel shows the distribution of $k^*$ computed with the ID estimated by 2NN and with ABIDE at termination.\\
    In the two bottom rows, we report ABIDE performance on the MNIST dataset, with sizes $n=70,000$ and $D=784$, together with a sample of 6 data points (top row).
    In the first panel, we observe how the ID drops from the initial value of 15.19 to 13.47. The second and third panels show the ID estimate at a fixed radius $t_B$ and at a fixed number of neighbours $k$, respectively. We notice that no plateaus are present and that p-values quickly become very low (in the case of fixed radius they are numerically 0).
    Conversely, ABIDE shows how the adaptive neighbourhood size allows for an ID estimate with a significantly larger p-value at an average scale $\frac{1}{n}\sum_{i=1}^n{t_{B,i}(k^*_i)}\approx 16$, corresponding to an average neighbourhood size of $\frac{1}{n}\sum_{i=1}^n{k^*_i}\approx 6.5$. Again, in the rightmost panel, one can observe the distributions of $k^*_i$ with 2NN and ABIDE.
    }
    \label{fig:opt_mnist}
\end{figure}

We begin by highlighting the common features of all three analyses. We observe that the ID always decreases from the initial 2NN estimate (Iteration 0 in the first panels). The reason is most likely related to the high sensitivity of 2NN to the noise, which unavoidably leads to overestimating the ID.
The same phenomenon occurs for BIDE estimates at small, fixed radii (second panels): even if the density is likely to be locally constant and, accordingly, the p-values are typically high, this is the regime where noise plays a relevant, if not dominant, role in ``smearing'' and hiding the true manifold ID. Unfortunately, when exploring the ID at larger $t_B$ values, that allow us to overcome the noise, we observe that the p-value drops consistently, meaning that for many data points the selected neighbourhood has significant density variations.

A similar behaviour occurs when the neighbour's size $k$ is fixed (third panels): for small $k$ the estimates are potentially reliable but are affected by noise; conversely, when the neighbourhood becomes larger, the density is less likely to be uniform.
In all cases we indeed observe a broad distribution of $k^*$ (rightmost panel), suggesting that, for this type of data, the ID estimate at fixed neighbourhood size (or even worse, at fixed radius) is inevitably poor.
Last, but not least importantly, we also observe the absence of a clear plateau in each of the BIDE scaling plots (second and third panels of each figure). This means that it is typically not possible to unequivocally decide which is the correct ID and the corresponding scale at which it is meaningful to observe the system.
Using ABIDE we are instead capable of automatically identifying a precise scale, where the ID estimate is close to the true one and is generally associated with a relatively high p-value. 
In the same fashion as what we did for the artificial dataset, we report ABIDE results 
as starred points in the ID/p-value vs $t_B$ and ID/p-value vs $k$ panels obtained with the canonical scaling algorithms of the BIDE estimator.

In particular, for the OptDigits we obtain ID=8.05 (with a p-value $\approx 10^{-14}$), starting from an ID=10.43 computed with 2NN. Concerning the neighbourhoods, we register $\frac{1}{n}\sum_{i=1}^n{k^*_i}\approx 9$.
For MNIST we have an ID decreasing from 15.19 to 13.47 (with p-value $\approx 10^{-55}$) and $\frac{1}{n}\sum_{i=1}^n{k^*_i}\approx 6.5$. Finally, for the peptide, we have an ID converging at 5.55 (p-value $\approx 10^{-25}$) starting from an initial value of 6.78 and $\frac{1}{n}\sum_{i=1}^n{k^*_i}\approx 10$.

\begin{figure}[!htb]
    \centering
    \includegraphics[width=1.0\linewidth]{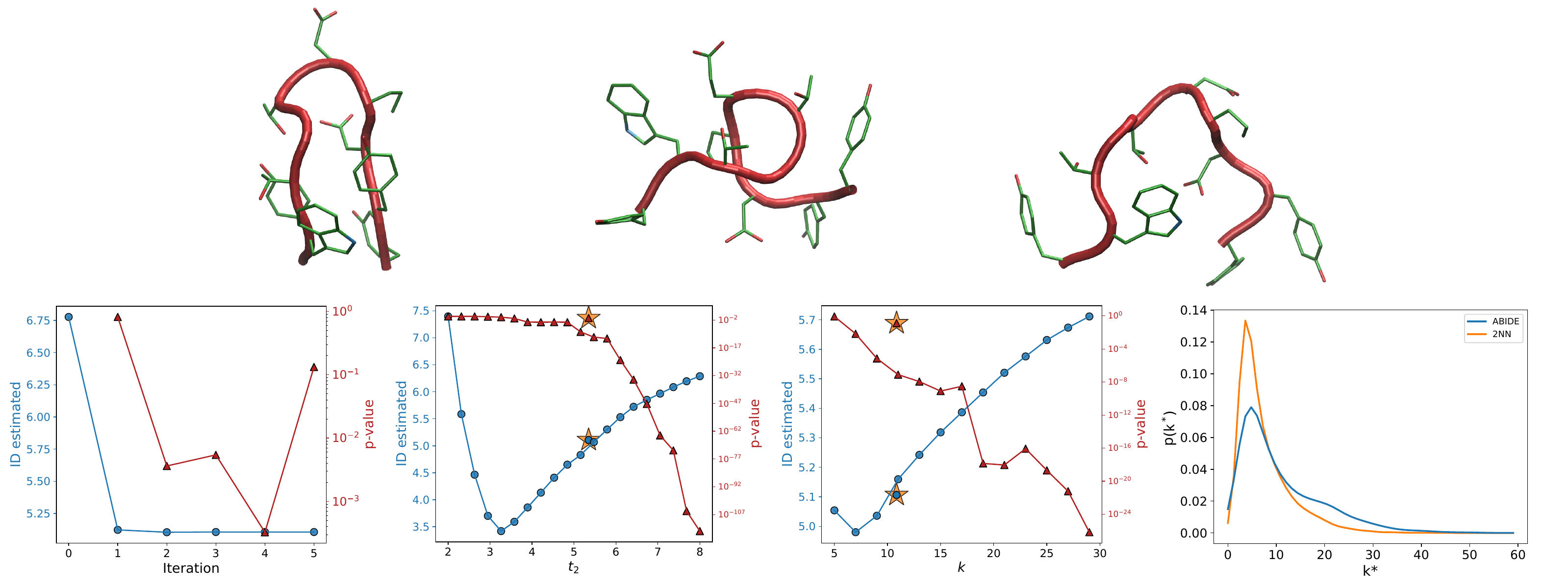}
    \caption{ABIDE performance on the dihedral representation of CLN025 peptide, a dataset of  size $n=3,758$, $D=32$. The first row shows three typical configurations of the peptide (from left to right): folded beta-hairpin, misfolded twisted, and unfolded.
    From the first panel of the second row, we can appreciate how the ID drops by 18\%, from the 2NN value of 6.78 to the ABIDE one of 5.55. BIDE estimates obtained by selecting the average ABIDE scale are close to the ones obtained with ABIDE  (second and third panels).
    }
    \label{fig:cln}
\end{figure}

\section{Discussion}
\label{sec:disc}
We introduced a new ID estimator aimed at mitigating some problems commonly affecting existing estimators: the choice of the scale at which the ID should be measured and the dependence of the ID estimate on the noise.
These two issues are strongly intertwined with the verification/enforcement of the statistical framework on which the ID estimator is based.
Starting from a likelihood formulation of the $k$-NN method, we leveraged an algorithm that finds, for each data point, the largest neighbours where the density is approximately constant. 
This, we have shown, allows us to enhance the statistical accuracy of ID estimates.
We remark that the theoretical validity of our method relies on the assumption that the data lie in a neighbourhood of a $C^1$ manifold of constant dimension $d$. Our method cannot be reliably used in pathological cases where the manifold has singularities or spatially varying intrinsic dimension.

To prove the effectiveness of this approach, which we called ABIDE, we tested it on both artificial and real-world datasets. 
The former, in which we had full control and comprehension of the data, allowed us to understand how ABIDE works and have an idea of its capabilities. 
In particular, we addressed its robustness against noise and verified the reliability of its estimates, exploiting the model validation developed in \cite{macocco}.
We then tackled different real-world datasets and, in each one of them, ABIDE managed to deliver insights that could not be obtained with other methods. 
In particular, in all of the cases we considered, the ID does depend on the scale. 
In these conditions, strictly speaking, the question ``What is the ID of this dataset?'' cannot be answered unequivocally. 
Using ABIDE, we can say that a meaningful scale at which the ID should be measured is the largest scale at which the density is constant. 
Clearly, in the presence of noise, the density will be constant even at lower scales, but the ID measured at such scales will be spuriously increased by the noise itself.

To check the potential of the approach introduced with ABIDE, it would be interesting to test the same framework on other suited ID estimators. This is briefly outlined in the Appendix for GRIDE (\citealp{denti}). Another estimator it would be interesting to consider is the maximum likelihood estimator proposed in \cite{levina2004maximum}.

The possibility of finding proper and meaningful neighbourhood sizes can be of great help in many relevant fields of data analysis. In particular, any method relying on $k$-NN statistics would benefit from the algorithm we have introduced in this paper. Our approach naturally extends to virtually any tool where the choice of the nearest-neighbour order to be considered is left to the user as a tuning parameter without being principle-based.
Among such methods, we mention density estimation techniques, where the adaptive neighbourhood was conceived (\citealp{rodriguez}) and has already found some applications (\citealp{carli2021statistically,carli2024densityestimationbinlessmultidimensional}); clustering routines, both those density-based (\citealp{d2021automatic}) and neighbours-graph-based (\citealp{10296014}); projective methods like Isomap (\citealp{isomap}), UMAP (\citealp{umap}) and Locally Linear Embedding (\citealp{lle}).
For all the aforementioned methods, the sole selection of an adaptive neighbourhood size might improve their performances.

A further extension of the proposed method concerns  the case of datasets with heterogeneous ID, similarly to the framework of \cite{Allegra2020} and \cite{denti2023bayesian}.
Finally, after we have grasped the capabilities of the ABIDE estimator in Euclidean spaces, it might be interesting to extend it to discrete metric spaces.

\section*{Acknowledgements}
A.M. acknowledges financial support from SNSF grant 208249.

\appendix
\section*{Appendix}

The following sections are devoted to the derivation of theoretical guarantees for the ABIDE estimator introduced in the paper.
Moreover, we briefly show how the methodology can be extended to its Bayesian counterpart. As a proof of concept, we show that GRIDE, another ID estimator recently proposed by \cite{denti}, can be improved by selecting optimal neighbourhoods of data points. In principle, this could be done for other ID estimators based on NNs.

\section{Theoretical guarantees}
We provide large-sample guarantees regarding the numerical convergence of the proposed algorithm (Proposition \ref{prop:numerical_convergence_new}) and show consistency of the resulting ABIDE estimator (Proposition \ref{prop:consistency}). Moreover, we show asymptotic normality (Proposition \ref{prop:normality}) of the estimator.
\paragraph{Notations and conventions.}
Throughout this section, we keep the notation introduced in the main text and use the following additional conventions.
For events, ``a.s.'' means almost surely.
For random sequences, $Z_n=o_P(1)$ means $Z_n$ converges to 0 in probability.
The Euclidean norm in $\mathbb{R}^D$ is denoted by $\|\cdot\|_2$, and for a set $A\subset\mathbb{R}^D$ we write $\operatorname{dist}(x,A):=\inf_{y\in A}\|x-y\|_2$.
The intrinsic volume measure on $M$ is denoted by $\mathrm{Vol}_M$.
For asymptotic statements regarding ABIDE, we write $d_n^*$ in place of $d^*$ to stress dependence on the sample size $n$.
When needed, we explicitly write $k_i^*(d)$ to emphasize dependence of the selected neighbourhood order on the current iterate $d$.

\begin{assumption}
\label{assump:manifold}
There exists a $d$-dimensional $C^1$ manifold 
$M \subset \mathbb{R}^D$ with $d < D$ such that
$$
X_i \in M_\epsilon := \{y \in \mathbb{R}^D : \operatorname{dist}(y,M) \le \epsilon\}
\quad \text{a.s.}
$$
for all $i=1,\dots,n$ and some $\epsilon > 0$ small enough.
\end{assumption}

\begin{assumption}
\label{assump:data}
The observations $X_1,\dots,X_n$ are iid random variables taking values in $\mathbb{R}^D$ and drawn from a common distribution $P_X$ that is absolutely continuous with respect to $\mathrm{Vol}_M$ on $M$.
\end{assumption}

We now derive the main results regarding the numerical convergence of Algorithm \ref{alg:adaptive}, as well as the consistency and asymptotic normality of the proposed ID estimator.
We start with the following proposition regarding the numerical convergence of Algorithm \ref{alg:adaptive}.

\begin{proposition}
\label{prop:numerical_convergence_new}
Under Assumptions \ref{assump:manifold}-\ref{assump:data}, for every tolerance $\delta>0$, Algorithm \ref{alg:adaptive} terminates with probability tending to 1 as $n\to\infty$.
\end{proposition}

\begin{proof}
    Let us consider the function $\widetilde{g}: \mathbb{R}_+ \to \mathbb{R}$ defined as $$d\mapsto \frac{\frac{1}{n}\sum_{i=1}^n{k_{A,i}^*(d)}}{\frac{1}{n}\sum_{i=1}^n{k_{B,i}^*(d)}}$$ and note that Algorithm \ref{alg:adaptive} is the fixed point iteration given by 
     \begin{equation}
     \label{eq:fixed_point_iteration}
         d_{m+1}=\frac{\log (\widetilde{g}(d_m))}{\log(g(d_m))}:=G(d_m)
     \end{equation}
     where $g:\mathbb{R}_+\to \mathbb{R}_+$ is the function given by $d\mapsto c_*^{1/d}$ and $G:\mathbb{R}_+ \to \mathbb{R}_+$.
    The aim is to employ a (version of) the fixed point theorem to show that the sequence of iterations \eqref{eq:fixed_point_iteration} converges to a fixed point.
    Let us consider a family of Borel sets $(\mathcal{G}_{h_i}^i)_h \subset [0,\infty)$ given by 
     $$\mathcal{G}_{h_i}^i=\{d\in [0,\infty): k_i^*(d)=h_i\},$$ 
     note that $(\mathcal{G}_{h_i}^i)_{h_i}$ constitutes a partition of $[0,\infty)$. Now, let $d_1\neq d_2$, clearly in the trivial case where for all $i=1,\dots,n$ we have $d_1, d_2 \in \mathcal{G}_{h_i}^i$ then the map $G$ is a contraction mapping because it is constant. On the other hand, when $d_1 \in \mathcal{G}_{h_i}^i$ and $d_2 \in \mathcal{G}_{l_i}^i$ with $h_i\neq l_i$ we have that $G$ is not guaranteed to be a contraction mapping, and we require a different argument.
     It is not restrictive to write the cumulative intensity on $B(x_i,r_{i,k_{A,i}^*})$ as $\rho_{A,i} V_{i,k_{A,i}^*}$ for some $\rho_{A,i}$, and the cumulative intensity on $B(x_i,r_{i,k_{B,i}^*})$ as
     $\rho_{B,i}V_{i,k_{B,i}^*}$
     for some $\rho_{B,i}$. Note that, by construction, $t_{A,i}(k_i^*) = \tau \, r_{i,k_i^*}$ and $t_{B,i}(k_i^*) = r_{i,k_i^*}$, and, by Assumption~\ref{assump:data}, since $k_i^* \le k_{\max}$, standard NN theory (Lemma 2.2 in \citealp{biau2015lectures}) gives $r_{i,k_i^*} \to 0$ in probability. Moreover, by Assumption~\ref{assump:manifold}, $\mathrm{Vol}_M(B(x, \tau r))/\mathrm{Vol}_M(B(x, r)) = \tau^d(1+o(1))$ as $r\to 0$. Therefore, it follows that
     \begin{align}
     \label{eq::tau}
     \frac{ \rho_{A,i} V_{i,k_{A,i}^*}}{ \rho_{B,i}V_{i,k_{B,i}^*}}
     = \tau^d\frac{\rho_{A,i}}{\rho_{B,i}}(1+o_P(1))
     = \tau^d\frac{\rho_{A,i}}{\rho_{B,i}} + o_P(1).
     \end{align}
     Now, since $k_{A,i}^*$ and $k_{B,i}^*-k_{A,i}^*$ are independent by construction, by explicit computation we have that 
     \begin{align*}
          p_{k_{A,i}^*|k_{B,i}^*}( x|y)&= \frac{p_{k_{A,i}^*}(x) p_{k_{B,i}^*-k_{A,i}^*}(y-x)}{p_{k_{B,i}^*}(y)}\\
          &=\frac{y!}{x!(y-x)!} \frac{(\rho_{A,i} V_{i,k_{A,i}^*})^x(\rho_{B,i}V_{i,k_{B,i}^*}-\rho_{A,i} V_{i,k_{A,i}^*})^{y-x}}{ (\rho_{B,i}V_{i,k_{B,i}^*})^{y}}\frac{(\rho_{B,i}V_{i,k_{B,i}^*})^{-x}}{(\rho_{B,i}V_{i,k_{B,i}^*})^{-x}}\\
          &=\binom{y}{x} \Big(\frac{\rho_{A,i}V_{i,k_{A,i}^*}}{\rho_{B,i}V_{i,k_{B,i}^*}}\Big)^{x}\Big(1-\frac{\rho_{A,i}V_{i,k_{A,i}^*}}{\rho_{B,i}V_{i,k_{B,i}^*}}\Big)^{y-x}.
     \end{align*}
     Therefore, using \eqref{eq::tau}, it follows that 
     $$
     k_{A,i}^*|k_{B,i}^* (d_1) \sim \mathrm{Binomial}\Big(k^*_{B,i}(d_1), \tau^d \frac{ \rho_{A,i}}{\rho_{B,i}}\Big), 
     $$
     with probability tending to 1, for some $d_1 \in [0,\infty)$.
     As a consequence, defining $\widetilde k^*(d_1):= \underset{i=1,\dots,n}{\max} \{k^*_{B,i}(d_1)\}+1$, the version of Hoeffding's inequality introduced in \cite{janson2004large} implies
    that on an event with probability at least $1-\gamma$,
    $$
   \Big|\frac{1}{n}\sum_{i=1}^n k_{A,i}^*|k_{B,i}^*(d_1)- \frac{1}{n} \sum_{i=1}^n k^*_{B,i}(d_1) \tau^{d} \frac{ \rho_{i,A}}{\rho_{B,i}}\Big|\leq C_\gamma \frac{\sqrt{\widetilde k^*(d_1)\sum_{i=1}^n k_{B,i}^*(d_1)^2} } {n},
   $$
   for every $\gamma>0$.
   Combining the inequality in the previous display with the fact that $k_i^*\leq k_{max}$ for all $i=1,\dots,n$, it follows that, with high probability, the average $\frac{1}{n}\sum_{i=1}^n k_{A,i}^*|k_{B,i}^*(d_1)$ lies in a shrinking neighbourhood around its expectation.
   Now, note that, as $n\to \infty$, by Assumption~\ref{assump:data} and the Lebesgue differentiation theorem, $ \rho_{A,i}/\rho_{B,i}\to 1$ in probability for all $i=1,\dots,n$. Thus, recalling that $\tau:= g(d_1)$, it follows
  \begin{align*}
      \widetilde{g}(d_1)&= \frac{\frac{1}{n}\sum_{i=1}^n k_{A,i}^*|k_{B,i}^*(d_1)}{\frac{1}{n}\sum_{i=1}^n k^*_{B,i}(d_1)}=\frac{\frac{1}{n} \sum_{i=1}^n k^*_{B,i}(d_1) \tau^{d} + o_P(1) }{\frac{1}{n}\sum_{i=1}^n k^*_{B,i}(d_1)}  = g(d_1)^d + o_P(1).
  \end{align*}
   Therefore, $\log  \widetilde{g}$ possesses the same asymptotic behaviour of $ d \log  g $, and $G$ converges uniformly in probability to a constant contraction mapping with fixed point equal to $d$.
   The desired result follows by employing (a version of) the fixed point theorem (see \citealp{kirk2002handbook}), which implies that the sequence of iterations \eqref{eq:fixed_point_iteration} converges to a fixed point with probability tending to 1.
\end{proof}

Since $G$ tends to a constant function, we expect that 2 iterations are enough to obtain stabilization of Algorithm \ref{alg:adaptive} at a given tolerance, and this is confirmed by the numerical experiments presented in the paper. Moreover, we remark that the regularity of the intensity of the Poisson process can improve the numerical convergence of the algorithm. In other words, when $h$ is locally well approximated by a constant function, termination of Algorithm \ref{alg:adaptive} holds for much smaller sample sizes.
Now we derive consistency for the estimator $d^*_n$.
\begin{proposition}
\label{prop:consistency}
Under Assumptions \ref{assump:manifold}-\ref{assump:data}, $d^*_n$ converges in probability to $d$ as $n\to \infty$.
\end{proposition}

\begin{proof}
    Using the arguments given in the proof of Proposition \ref{prop:numerical_convergence_new}, it follows that
    for any $\epsilon>0$ and $\delta >0$, as $n\to\infty$,
     \begin{align*}
         P\Big(|d_n^*-d|>\epsilon\Big)&\leq P\Big(\{|G(d_n^*)+\delta-d|>\epsilon\}\Big)=P\Big(\Big|\frac{\log (\widetilde{g}(d_n^*))}{\log(g(d_n^*))}+\delta -d\Big|>\epsilon\Big)\\
         &= P\Big(\Big|\frac{ d \log ({g}(d_n^*))}{\log(g(d_n^*))}+\delta-d\Big|>\epsilon\Big)+o(1).
     \end{align*}
It suffices to take $\delta \to 0$ and the thesis follows.
\end{proof}

The next Proposition is a qualitative result concerning the asymptotic normality of the proposed estimator.
\begin{proposition}
\label{prop:normality}
     Under Assumptions \ref{assump:manifold}-\ref{assump:data}, there exist real constants $c,C$ such that $Z_n=\sqrt{n c I(d^*_n)}(d^*_n-d+C)$ converges in distribution to $Z \sim \mathcal{N}(0,1)$ as $n\to \infty$.
\end{proposition}
\begin{proof}
Owing to Proposition \ref{prop:numerical_convergence_new} $d^*_n$ is bounded in probability, therefore, asymptotic normality follows directly from the classical properties of M-estimators \citep{van2000asymptotic} combined with the version of the Central Limit Theorem for dependent random variables of \cite{baldi1989normal}.
\end{proof}

Numerical experiments show that the correction constants $c, C$ (due to dependencies and curvature) can be taken equal to 1 with no significant loss of accuracy when approximating the limiting distribution; see Figure \ref{fig:normality}.

\section{Bayesian version of ABIDE}
The Bayesian version of the proposed approach is a straightforward extension of the Bayesian estimator presented in the supplementary material of \cite{macocco}. Here we briefly outline it and show how the adaptive version can be obtained. By a conjugacy argument, consider the following prior, $p=\tau ^d\sim \mathrm{Beta}(\alpha_0,\beta_0)$, 
and derive the corresponding posterior distribution
\begin{equation*}
    p|k_{B,1},\dots, k_{B,n} \sim \mathrm{Beta}(\alpha, \beta)
\end{equation*}
where $\alpha=\alpha_0+\sum_{i=1}^nk_{A,i}$ and $\beta =\beta_0+\sum_{i=1}^n (k_{B,i}-k_{A,i})$. In our adaptive framework, we have instead the following posterior distribution
\begin{equation*}
    p|k_{B,1}^*,\dots, k_{B,n}^* \sim \mathrm{Beta}(\alpha^*, \beta^*)
\end{equation*}
where $\alpha^*=\alpha_0+\sum_{i=1}^nk_{A,i}^*$ and $\beta^* =\beta_0+\sum_{i=1}^n(k_{B,i}^*-k_{A,i}^*)$. Therefore, omitting the conditioning on $k_{B,1}^*,\dots, k_{B,n}^*$ for notational convenience, the posterior density of the ID, derived by a change of variable argument, is proportional to 
\begin{equation*}
    f_d(x) \propto  \tau^{x(\alpha^*-1)}(1-\tau^x)^{\beta^*-1} |\tau ^x \log \tau|, 
\end{equation*}
where the posterior expectation and the posterior variance can be derived using the well-known fact that for $X\sim \mathrm{Beta}(\alpha,\beta)$ we have 
$$\operatorname{E}[X]=\psi_0(\alpha)-\psi_0(\alpha+\beta),\quad \operatorname{Var}[X]={\psi_1(\alpha)-\psi_1(\alpha+\beta)},$$
where $\psi_0(z)=\frac{d}{dz}\log \Gamma(z)$ and $\psi_1(z)=\frac{d^2}{dz^2}\log \Gamma(z)$ are the digamma and the trigamma functions respectively.
Thus, we obtain the posterior expectation and the posterior variance
\begin{equation}
\label{eq:post-mean-var}
    \operatorname{E}[d] 
    =\frac{\psi_0(\alpha^*)-\psi_0(\alpha^*+\beta^*)}{\log \tau}, \quad \operatorname{Var}[d]=\frac{\psi_1(\alpha^*)-\psi_1(\alpha^*+\beta^*)}{(\log\tau)^2}.
\end{equation}
Therefore, adopting a squared error loss function, $\operatorname{E}[d]$ is a proper Bayesian estimator of the ID and, it gives rise to a Bayesian version of Algorithm \ref{alg:adaptive} obtained setting $d_\text{next}=\operatorname{E}[d]$. This is named the \emph{Bayesian Adaptive Binomial ID Estimator} (BABIDE). Now we show that BABIDE inherits the theoretical guarantees of ABIDE.

\begin{proposition}
    \label{prop:BABIDE}
    BABIDE is asymptotically equivalent to ABIDE.
\end{proposition}
\begin{proof}
    Using the well-known asymptotic approximation of $\psi_0$ given by 
    $$
        \psi_0(z) \sim \log(z) - \sum_{j=1}^\infty \frac{\zeta(1-j)}{z^j}
    $$
    where $\zeta$ is the Riemann Zeta function, we obtain the asymptotic equivalence 
    $$
        \operatorname{E}[d]\sim \frac{\log(\sum_{i=1}^n k_{A,i}^*)-\log(\sum_{i=1}^nk_{B,i}^* )}{\log \tau},
    $$
    which concludes the proof.
\end{proof}

As a consequence of Proposition \ref{prop:BABIDE}, algorithm termination is naturally inherited by Proposition \ref{prop:numerical_convergence_new}.
    Moreover, note that statistical guarantees regarding consistency and asymptotic normality can be directly extended to BABIDE by a posterior concentration argument, and using the classical Bernstein–von Mises theorem (\citealp{lecam}) for the posterior distribution.

\section{Adaptive Generalized Ratios ID Estimator}
Another approach that fits into the same theoretical framework outlined in the paper is proposed in \cite{denti}. 
Let $r_{i,l}$ be the distance between unit $i$ and its $l$-th NN according to a selected metric. Then, let $n_1, n_2$ be two positive integers with $n_2>n_1$ and define the ratio $\mu_{i,n_1,n_2}= r_{i,n_2}/r_{i,n_1}$ which is a.s.\ well-defined for continuous data. For discrete data, duplicated points leading to zero values at the denominator of $\mu_i$, need to be removed. The authors derive the density function of $\mu_{i,n_1,n_2}$:
\begin{equation}
 \label{eq:gride-density}
    f_{\mu_{i,n_1,n_2}}(\mu)
    =\frac{d(\mu^d-1)^{n_2-n_1-1}}{\mu^{d(n_2-1)+1} \mathrm{B}(n_2-n_1,n_1)},  \quad \mu>1
\end{equation}
where $\mathrm{B}(\cdot,\cdot)$ is the Beta function and $d$ is the ID.
In the following, we omit $n_1$ and $n_2$ from the subscript for notational convenience. 
The Authors introduce the \emph{Generalized Ratios ID Estimator} (GRIDE) given by
\begin{align}
\label{eq:gride}
          \widehat d  &= \arg\underset{d\geq 0}{\max}\, \log L(d)= \arg\underset{d\geq 0}{\max}\,\log \prod_{i=1}^n \frac{d(\mu_i^d-1)^{n_2-n_1-1}}{\mu_i^{d(n_2-1)+1} \mathrm{B}(n_2-n_1,n_1)},
\end{align}
where the maximization is carried out numerically.
This approach generalizes the 2NN estimator proposed in \cite{facco2017}, since when $n_1=1$ and $n_2=2$, the density in 
\eqref{eq:gride-density} 
reduces to a $\mathrm{Pareto}(1,d)$ and
$
    \widehat{d}=n/\sum_{i=1}^n \log \mu_i,
$
which is equivalent to the estimator proposed in \cite{facco2017}.
It must be pointed out that, when compared to the GRIDE estimator based on values of  $n_2 > 2$, the 2NN is more likely to satisfy the local homogeneity assumption of the Poisson process 
density because it only uses information up to the scale of the second NN of each statistical unit. On the other hand, it is potentially affected by the presence of noise in the data.
The selection of $n_1$ and $n_2$ is somehow linked to the selection of the radii in the framework of \cite{macocco}, where the choice of the radii induce $k_{A,1}\dots,k_{A,n},k_{B,1}\dots,k_{B,n} $ 
in estimator \eqref{eq:bide}, and this choice is affected by the same trade-off between robustness to noise and local homogeneity.
The estimator in \eqref{eq:gride} is straightforwardly extended to its adaptive version by 
splitting the likelihood in its single unit contributions, plugging in the unit dependent $n_{2,i}=k^*_i$ and, as suggested in \cite{denti}, taking $n_{1,i}= n_{2,i}/2$. This leads to an iterative algorithm similar to ABIDE. We might refer to it as \emph{Adaptive GRIDE} (AGRIDE). Some simulations showed a comparable behaviour of ABIDE and AGRIDE. However, it should be recalled that ABIDE is based on BIDE, which has a closed-form expression and does not require additional numerical routines to be computed. This is a relevant advantage as it makes the adaptive version mathematically more convenient when deriving theoretical guarantees.

\section{Additional experiments}
\subsection{Behaviour of adaptive neighbourhoods}

Figure \ref{fig:2nn_vs_abide} reports a comparison of the distributions of $k^*_i$ (left) and corresponding radii $t_{B,i}(k_i^*)$ (right) obtained using the ID estimated with the 2NN estimator (blue) against those obtained with ABIDE (orange) in reference to Panel E of Figure~\ref{fig:moebius1}. The associated IDs are $\approx 5$ for 2NN and $\approx 2.35$ for ABIDE. One can appreciate how the neighbourhoods become larger together with the lowering of the ID.
This shows how a proper estimation of the ID also affects, non-trivially, the estimation of $k^*_i$ and vice-versa.
The right panel displays the distribution of the distances associated with the selected neighbours. 
One can appreciate how also this distribution is broad, with a long tail and two distinct peaks at small distances. 
The shape of such distributions allows us to understand why the results obtained at fixed uniform radius $t_{B,i}=t_B$ or at fixed uniform neighbourhood extension $k_i=k$ could not provide reliable results for this kind of dataset, where the density is rapidly changing in a non-uniform way across the dataset.

Figure \ref{fig:kstar_vs_stuff} shows the evolution of $\overline{k^*}=\frac{1}{n}\sum_{i=1}^n k_i^*$ (top) and corresponding $\overline{t_A(k^*)}=\frac{1}{n}\sum_{i=1}^n t_{A,i}(k_i^*)$ (bottom) computed on multivariate Gaussian data. In each panel, the parameters that do not vary are fixed at $n=10,000$, $d=5$ 
and $\alpha=0.01$. It is at once apparent that $\overline{k^*}$ scales sub-linearly with the sample size (top-left), and $\overline{t_A(k^*)}$ decreases as $n$ increases (bottom-left), making the neighbourhoods asymptotically smaller and smaller. In the central panels, we generate data with different ID values and compute $\overline{k^*}$ and $\overline{t_A(k^*)}$ using the true ID. It is clear that $\overline{k^*}$ decreases as the ID increases since density variations happen at smaller neighbourhoods, while $\overline{t_A(k^*)}$ increases because higher dimensions imply higher distance between data points. Finally, on the right panels, we show that a larger $\alpha$ leads to smaller neighbourhoods ensuring higher protection against type 2 errors.

\begin{figure}[!htb]
    \centering
    \includegraphics[width=1.0\linewidth]{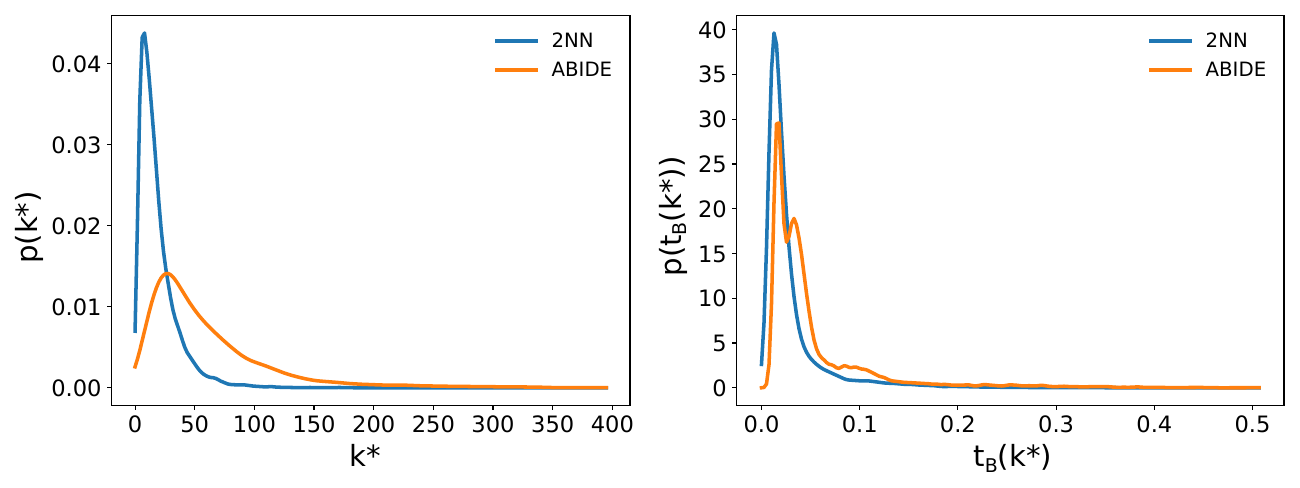}
    \caption{M\"obius dataset: comparison of the distributions of $k^*_i$ (left) and corresponding radii $t_{B,i}(k_i^*)$ (right) obtained using the ID estimated with the 2NN estimator (blue) against those obtained with ABIDE (orange) in reference to Panel E of Figure~\ref{fig:moebius1}.
    }
    \label{fig:2nn_vs_abide}
\end{figure}

\begin{figure}[!htb]
    \centering
    \includegraphics[width=1.0\linewidth]{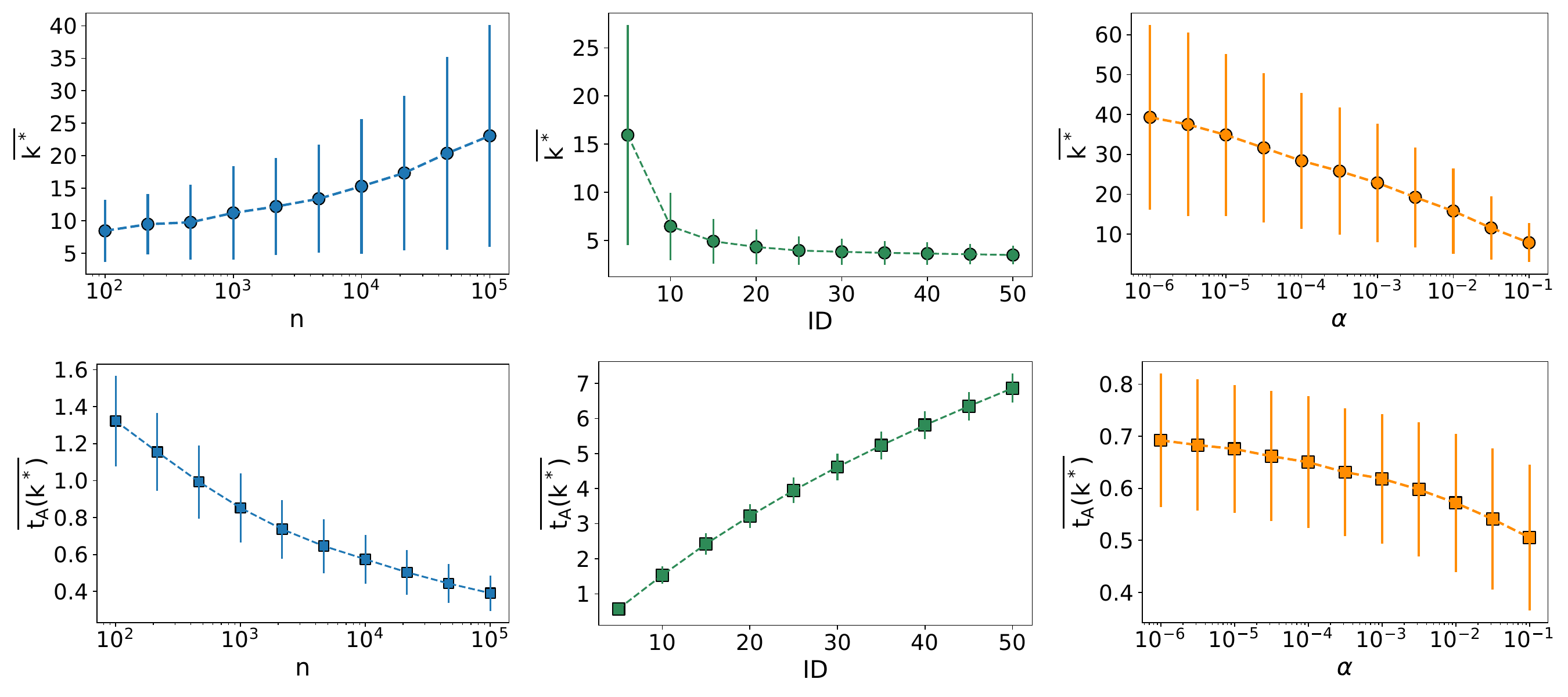}
    \caption{Evolution of $\overline{k^*}=\frac{1}{n}\sum_{i=1}^n k_i^*$ (top) and corresponding $\overline{t_A(k^*)}=\frac{1}{n}\sum_{i=1}^n t_{A,i}(k_i^*)$ (bottom) computed on multivariate Gaussian data. The vertical bars represent the standard deviation of the distributions.
    In each panel, the parameters that do not vary are fixed at $n=10,000$, $d=5$, 
    and $\alpha=0.01$.
    }
    \label{fig:kstar_vs_stuff}
\end{figure}

\subsection{Large sample distribution}
Figure \ref{fig:normality} reports the Monte Carlo distribution of the centred, re-scaled and normalized statistic $\sqrt{n I(d^*_n)}(d^*_n-d)$ for $n=10,000$ based on 500 Monte Carlo independent replicas of a 5-dimensional uniform distribution with periodic boundary conditions. From the figure it is at once apparent that the distribution of $\sqrt{n I(d^*_n)}(d^*_n-d)$ tends to be Gaussian, thus confirming the results of Proposition \ref{prop:normality}.

\begin{figure}[!htb]
    \centering
    \includegraphics[width=0.6\linewidth]{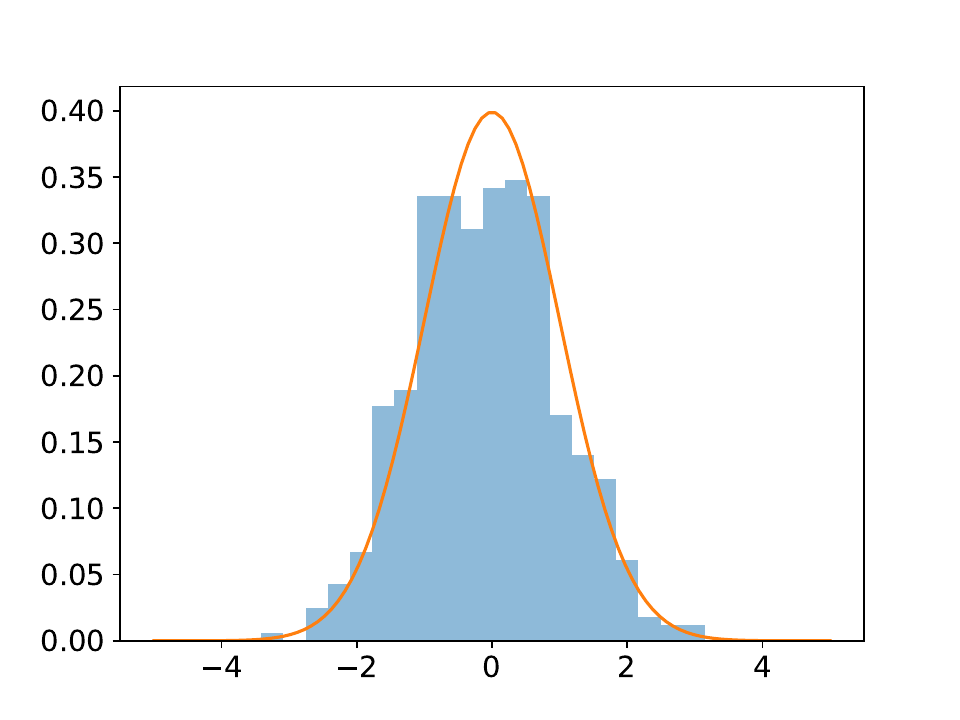}
    \caption{
    Monte Carlo distribution of the centred, re-scaled and normalized statistic $\sqrt{n I(d^*_n)}(d^*_n-d)$ for $n=10,000$ based on 500 Monte Carlo independent replicas of a 5-dimensional uniform distribution with periodic boundary conditions.
    }
    \label{fig:normality}
\end{figure}

\subsection{Selection of the testing threshold and results on execution times}
\label{sec:bonferroni}
One approach to selecting the testing threshold $D_{\text{thr}}$ is to use a Bonferroni-adjusted $(1-\alpha)$ quantile of the $\chi^2_1$ distribution.
Denoting by $h$ the number of sequential tests performed for each data point, we compare the following four choices: 
 $\alpha= 0.01$,  $\alpha= \frac{0.01}{h}$, $\alpha= \frac{0.01}{n}$, and $\alpha= \frac{0.01}{nh}$.
We consider the same setting of Figure~\ref{fig:simul1_noise}, i.e., points sampled from a 2-dimensional Gaussian distribution, embedded in 100 dimensions with added noise.
Results are reported in Figure~\ref{fig:bonferroni}, where we show the ID estimate as a function of the number of data points $n$ (left), and of the noise intensity (right), for the four different choices of $\alpha$. 
In practice, we observe these corrections to be relatively negligible in terms of ID estimates when the sample size is sufficiently large.
However, as a general guideline for practitioners, if there is suspicion of substantial noise, we suggest setting a larger threshold $D_{\text{thr}}$ which corresponds to a smaller $\alpha$, e.g.\ inversely proportional to the number of data points.
\begin{figure}[!htb]
    \centering
    \includegraphics[width=0.45\linewidth]{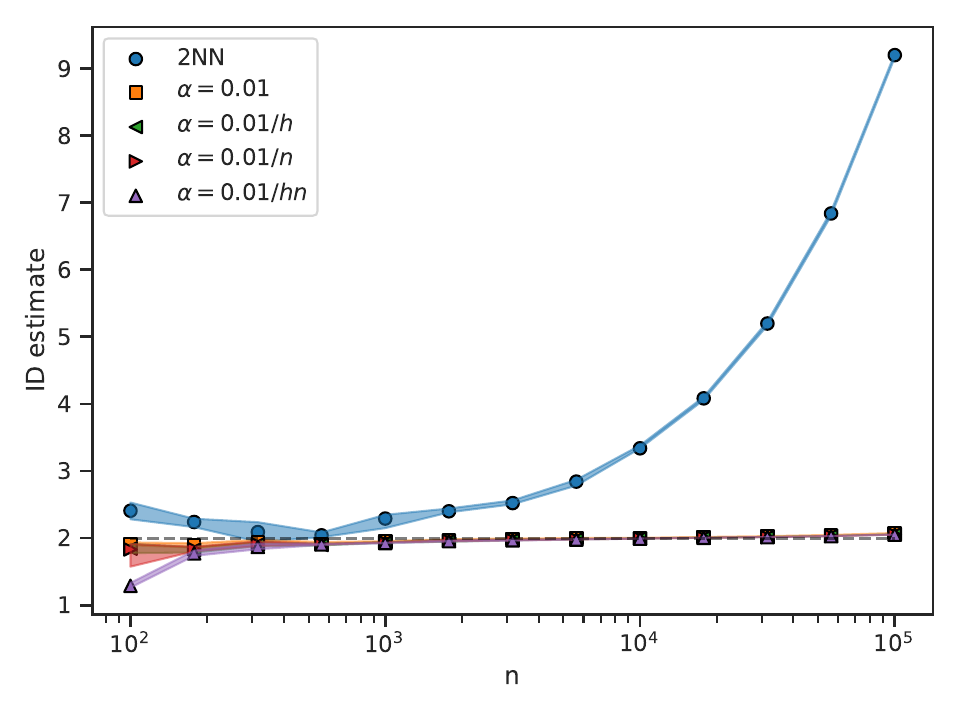}
    \includegraphics[width=0.45\linewidth]{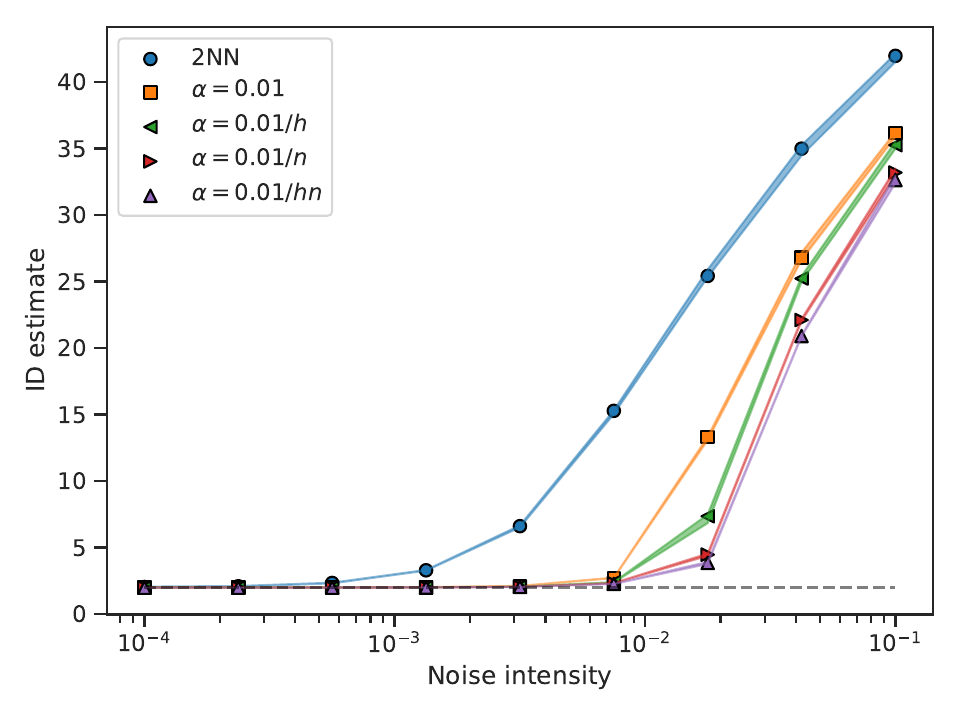}
    \caption{Effect of different Bonferroni correction methods on the ID estimates for a 2 dimensional Gaussian embedded in 100 dimensions. \textbf{Left}: ID estimates as a function of the number of data points with the noise fixed to $\sigma_\epsilon=10^{-3}$. \textbf{Right}: ID estimates as a function of the noise with number of point fixed to $n=5000$.}
    \label{fig:bonferroni}
\end{figure}

In Table~\ref{tab:exec_times} we report the execution times for the noisy Gaussian distribution (Section~\ref{sec:noisy_gaussian}) and for the M\"obius strip (Section~\ref{sec:moebius}). It is evident from Table~\ref{tab:exec_times} that the nearest-neighbour computation dominates the execution time, being on average five times slower than the 5-iteration ABIDE workflow.

\begin{table}[htb]
    \centering
    \begin{tabular}{l||c|c|c|c|c}
        \toprule
        Dataset $(n \times D)$& NN & $\alpha=0.01$ & $\alpha=\frac{0.01}{h}$ & $\alpha=\frac{0.01}{n}$ & $\alpha=\frac{0.01}{nh}$ \\
        \midrule
        M\"obius strip $(20k \times 50)$  & 5.3 & 0.22 & 0.27 & 0.36 & 0.51 \\
        Gaussian $(5k\times100)$ & 2.5 & 0.10 & 0.11 & 0.12 & 0.17 \\
        \bottomrule
    \end{tabular}
    \caption{Running times (in seconds) across two datasets for a single 
    i7-1165G7 core. The values reported in the columns with $\alpha$ are intended for a single iteration step, comprehensive of both $k^*$ and ID calculation.}
    \label{tab:exec_times}
\end{table}

\subsection{Comparison with other NN-based ID estimators}
\label{sec:comparison}
We compare the performance of both ABIDE and AGRIDE against the other most used kNN estimators, namely DANCo~\citep{ceruti2014danco}, MLE~\citep{levina2004maximum}, TLE~\citep{amsaleg2019intrinsic}, and ESS~\citep{johnsson2014low}.
We reproduce the experiment of Panel F of Figure~\ref{fig:moebius1}, namely a 2-dimensional, highly non-uniform manifold, which is twisted and then embedded in 50 dimensions. The noise is then added to all the 50 dimensions. 
The experiments were carried out using the Scikit-dimension~\citep{Bac2021} package, setting the number of considered neighbours to 20.
\begin{figure}
    \centering
    \includegraphics[width=0.6\linewidth]{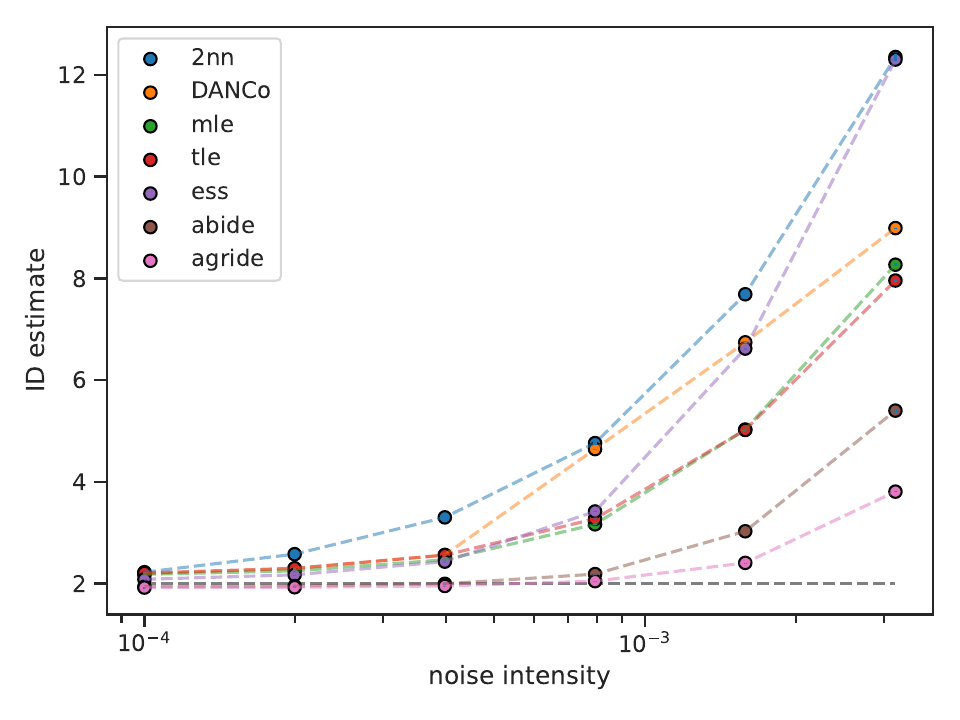}
    \caption{ID estimates as a function of the noise added to a 2-dimensional manifold twisted and embedded in 50 dimensions. We report the results for the most commonly used NN-based ID estimators.}
    \label{fig:comparison}
\end{figure}
From Figure~\ref{fig:comparison}, ABIDE and AGRIDE are manifestly superior in resisting to the noise and, thus, providing a better ID estimate.

\bibliographystyle{apalike}
\bibliography{bib}

\end{document}